\documentclass[a4paper,smallcondensed,natbib]{svjour3}

\usepackage{mathtools}
\usepackage{amssymb}
\usepackage{tikz}
\usepackage{cite}
\usepackage{lmodern}
\usetikzlibrary{arrows,positioning,decorations.pathmorphing}
\usepackage[colorlinks=true,urlcolor=blue,citecolor=blue,linkcolor=blue,pdfborder=0 0 0]{hyperref}
\usepackage{doi}
\usepackage{orcidlink}

\newcommand{\newthought}{\bigskip\noindent{}}

\let\originalleft\left
\let\originalright\right
\renewcommand{\left}{\mathopen{}\mathclose\bgroup\originalleft}
\renewcommand{\right}{\aftergroup\egroup\originalright}

\spnewtheorem{assumption}{Assumption}{\bfseries}{\rmfamily}

\begin{document}

\title{On the reusability of samples in active learning}
\author{Gijs van Tulder\,\orcidlink{0000-0003-1635-5423} \and Marco Loog\,\orcidlink{0000-0002-1298-8461}}
\authorrunning{Gijs van Tulder \and Marco Loog}
\institute{Gijs van Tulder \at
           Data Science group, Faculty of Science, Radboud University, Nijmegen, The Netherlands 
           \and
           Marco Loog \at
           Pattern Recognition Laboratory, Delft University of Technology, Delft, The Netherlands 
}
\date{%
  June 2022%
}

\maketitle

\begin{abstract}
  An interesting but not extensively studied question in active learning is that of sample reusability: to what extent can samples selected for one learner be reused by another?
  This paper explains why sample reusability is of practical interest, why reusability can be a problem, how reusability could be improved by importance-weighted active learning, and which obstacles to universal reusability remain.
  With theoretical arguments and practical demonstrations, this paper argues that universal reusability is impossible.
  Because every active learning strategy must undersample some areas of the sample space, learners that depend on the samples in those areas will learn more from a random sample selection.
  This paper describes several experiments with importance-weighted active learning that show the impact of the reusability problem in practice.
  The experiments confirmed that universal reusability does not exist, although in some cases -- on some datasets and with some pairs of classifiers -- there is sample reusability.
  Finally, this paper explores the conditions that could guarantee the reusability between two classifiers.
\end{abstract}

\newpage
\raggedbottom

\section{Introduction}

Active learning is useful if collecting unlabelled examples is cheap but labelling those examples is expensive.
An active learning algorithm looks at a large set of unlabelled examples and asks an oracle to label the examples that look interesting.
Labelling can be expensive -- it may involve asking a human or doing an experiment -- so a good active learner will try to reduce the number of labels it requests, by querying only those examples that it expects will lead to the largest improvements of the model.
By not labelling useless examples, active learning can learn a better model at a lower labelling cost.

Active learning is an iterative process that optimises the sample selection for a particular classifier.
In each iteration, the current set of samples is used to train an intermediate classifier, which is then used to select a new sample to add to the training set (Figure~\ref{fig:al-cycle}).
For example: uncertainty sampling -- one of the first, simple active learning strategies \citep{Lewis1994} -- picks the example that is closest to the current decision boundary.
Newer strategies can be more advanced, but they still optimise their selection for a specific classifier.

\begin{figure}[h]
  \centering
    \begin{tikzpicture}[scale=0.4\textwidth, line width=0.3, >=stealth]
      \node (dummy) {};
      \node[above=0.8cm of dummy] (a) { train classifier };
      \node[above=0.8cm of a] (q) { initial training set };
      \node[right=0.1cm of dummy] (b) { select next example };
      \node[below=0.8cm of dummy] (c) { ask oracle for label };
      \node[left=0.1cm of dummy]  (d) { add to training set };
      
      \path[->, shorten <= 2pt, shorten >= 2pt]
         (a.0)  edge [bend left=40] (b.30)
         (b.-30) edge [bend left=40] (c.0)
         (c.180)  edge [bend left=40] (d.210)
         (d.150) edge [bend left=40] (a.180);
      \path[->, shorten <= 2pt, shorten >= 2pt]
         (q.south) edge (a.north);
    \end{tikzpicture}
  \caption{
    Active learning is an iterative process; most algorithms use the samples selected in a previous iteration to decide on the next sample.
  }
  \label{fig:al-cycle}
\end{figure}
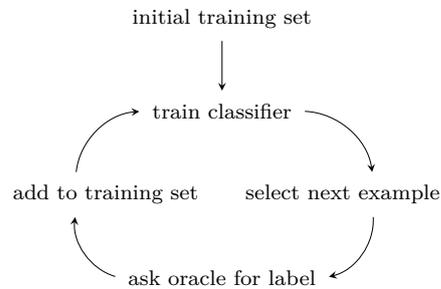

As a result, active learning will create a sample selection that is biased towards the classifier that was used to select the samples.
In some cases, however, the final classifier might differ from the classifier that was used during selection.
For example, if labelling samples is very expensive, the data has to be reusable for many different applications -- perhaps even applications that are not yet known when the selection is made.
Efficiency can be another reason: the final model may be too complex to retrain it for each selection step, in which case a simpler, faster model may be used for the selection phase \citep[e.g.,][]{Lewis1994a}.
And sometimes the best model for the problem can only be chosen after the examples have been selected: in natural language processing, for example, it is not uncommon to collect and label the data before choosing a language model \citep{Baldridge2004}.

If the active selection is used for a different classifier than the one used during selection, we might expect a lower performance than if the same classifier is used for both phases.
But at least we might ask that active learning does not reduce the performance, i.e., that it is not worse than using randomly selected samples.
This the definition of sample reusability: there is no sample reusability if a classifier that was not used to select the samples, performs worse if trained with the active selection than it would if trained on a random selection.

Ideally, active learning should produce sample selections that can be reused by different classifiers.
Unfortunately, the sample selection is biased in favour of a particular classifier.
This bias reduces the number of labels required to train the classifier, but it can also make the sample selection less reusable.
Furthermore, the bias can even introduce problems if the same classifier is used for selection and training, because most classifiers expect their training samples to be unbiased, and do not perform well on strongly biased data \citep{Zadrozny2004,Fan2007}.

Importance-weighted active learning (IWAL) is a recent active learning strategy that aims to correct the bias in the sample selection.
It combines a clever sample selection strategy with importance weighting, which allows it to provide unbiased estimators for the prediction error.
This can improve the results for any active learning application, and it might also make the sample selection more reusable by other classifiers.
In fact, the authors of the importance-weighted active learning framework claim that because the estimates are unbiased, the algorithm \textit{will} indeed produce a reusable sample selection \citep{Beygelzimer2010,Beygelzimer2011}.\footnote{%
  \raggedright%
    The reusability of importance-weighted active learning was also mentioned in the presentation ``Efficient Active Learning'' by Nikos Karampatziakis at the ICML 2011 Workshop on On‐line Trading of Exploration and Exploitation 2.
    This presentation is available at \url{http://videolectures.net/explorationexploitation2011_karampatziakis_efficie/} as a video and at \url{http://lowrank.net/nikos/pubs/exp2al.pdf} as slides.
    Another, undated presentation titled ``Active Learning via Reduction To Supervised Classification'' by John Langford et al., available at \url{http://hunch.net/~jl/projects/interactive/aalwoc.pdf}, lists sample reusability as one of the ``fringe benefits'' of importance-weighted active learning.
}

So far, there is no clear answer to why, when and whether there is sample reusability in active learning.
Early active learning papers occasionally mention the problem -- for example, \citet{Lewis1994a} discuss ``heterogeneous uncertainty sampling'' -- but overall sample reusability has received little attention.
In the most extensive study, \citet{Tomanek2011} tested several hypotheses with different classifiers and datasets.
The results were inconclusive: sample reuse was sometimes successful and sometimes was not, and there were no classifier combinations that always performed well.
None of the hypotheses, such as ``similar classifiers work well together'', could be confirmed.
\citet{Baldridge2004} and \citet{Hu2011} found similar results.
These studies share an important limitation: they only studied uncertainty sampling.
Sample reusability is probably not independent of the selection strategy.
Uncertainty sampling has well-known problems, because it creates a strongly biased sample selection and does nothing to reduce this bias \citep{Dasgupta2008}.

Importance-weighted active learning provides an unbiased estimates of the sample density, which may solve most of the problems caused by the bias.
But the unweighted sampling density of importance-weighted active learning is still biased, because the algorithm undersamples some areas of the sample space.
We argue that this may reduce the reusability of the sample selection.

This paper explores these topics in more detail.
Section~\ref{sec:sample-reusability} gives a definition of sample reusability and related concepts.
Section~\ref{sec:iwal} introduces the algorithm for importance-weighted active learning and explains how this algorithm solves some of the bias-related problems in active learning.
Section~\ref{sec:sample-reusability-in-iwal}, however, argues that importance weighted active learning does not completely solve the sample reusability problem.
We illustrate this with an experiment on a synthetic problem in Section~\ref{sec:experiment-synthetic} and with experiments on several public datasets in Section~\ref{sec:practice}.
Finally, Section~\ref{sec:conditions} discusses conditions that could guarantee sample reusability.

\section{Sample reusability}
\label{sec:sample-reusability}

Active learning works in two phases: first the selection strategy selects and labels examples -- perhaps training some intermediate classifiers to help in the selection -- and produces a set of labelled examples.
The classifier that is used in this phase is the \textit{selector}.
In the second step, the collected examples are used to train a final classifier that can be used to classify new examples.
This classifier is the \textit{consumer}.
The selector and the consumer do not have to be classifiers of the same type.
Perhaps the samples will be used to train a different type of classifier, or perhaps the samples will be used to solve a different type of problem, such as regression or structured learning.

Using a consumer that is different from the selector is \textit{sample reuse}: reusing the sample selection for something for which it was not selected.
\citet{Tomanek2010} gives a more formal definition:

\begin{definition}[Sample reuse]
  Sample reuse describes a scenario where a sample $S$ obtained by active learning using learner $T_1$ is exploited to induce a particular model type with learner $T_2$ with $T_2 \neq T_1$.
\end{definition}

Sample reuse is always possible, but the question is how well it works.
It is acceptable if a classifier performs slightly better if it is trained on samples selected for that classifier than if it is trained on samples selected for a different classifier.
But at least reusing the active selection should give a better performance than training on a random sample selection: if sample reuse is worse than random sampling, it would be better not to use active learning at all.
Therefore, we speak of \textit{sample reusability} if we expect that the consumer will learn more from the samples selected by the selector than from a random sample selection.
This is reflected in the definition of sample reusability by \citet{Tomanek2010}:
\begin{definition}[Sample reusability]
  Given a random sample $S_{RD}$, and a sample $S_{T_1}$ obtained with active learning and a selector based on learner $T_1$, and a learner $T_2$ with $T_2 \neq T_1$.
  We say that $S_{T_1}$ is reusable by learner $T_2$ if a model $\theta'$ learned by $T_2$ from this sample, i.e., $T_2\left(S_{T_1}\right)$, exhibits a better performance on a held-out test set $\mathcal{T}$ than a model $\theta''$ induced by $T_2\left(S_{RD}\right)$, i.e., $\mathrm{perf}\left(\theta',\mathcal{T}\right) > \mathrm{perf}\left(\theta'',\mathcal{T}\right)$.
\end{definition}
\noindent
Note that this definition of sample reusability is subject to chance.
It depends on the initialisation and samples presented to the algorithm.
In this paper, sample reusability means \textit{expected} sample reusability: does the algorithm, averaged over many runs, perform better with active learning than with random sampling?
That is, there is sample reusability if
\[
  \mathrm{E}\left[\mathrm{perf}\left(\theta',\mathcal{T}\right)\right] > \mathrm{E}\left[\mathrm{perf}\left(\theta'',\mathcal{T}\right)\right]
  \text{.}
\]

Sample reusability may be difficult to achieve if the selector and consumer require different samples.
The sample selection of active learning is biased towards samples that are useful for the selector.
If the consumer requires different samples, these samples will be underrepresented in the selection by active learning and the consumer might learn more from an unbiased random selection.
Correcting the bias is an important step to improve the sample reusability.

\section{Importance-weighted active learning}
\label{sec:iwal}

Importance-weighted active learning \citep{Beygelzimer2009} is a recent active learning strategy that is designed to reduce the bias-related problems of earlier strategies, by combining importance weighting with a biased random selection process.
Importance weighting helps to correct the bias, while the randomness in the selection ensures that the does not systematically exclude any area of the sample space.
These two properties make it possible to prove that importance-weighted active learning will converge to the same solution as random sampling \citep{Beygelzimer2010}.

Because importance-weighted active learning avoids and corrects some of the bias in the sample selection, it might also, perhaps, improve the sample reusability.
This idea from the importance-weighted active learning paper \citep{Beygelzimer2010,Beygelzimer2011} will be discussed in Section~\ref{sec:bias-correction} and further.

\subsection{Algorithm}
\label{sec:iwal-algorithm}

The importance-weighted active learning algorithm (Figure~\ref{fig:basic-iwal-algorithm}) is a sequential active learner.
It considers each example in turn and decides immediately if this new example should be labelled or not.
It uses a ``biased coin toss'' with a selection probability $P_x$ that defines the probability that example $x$ will be selected.
If an example is selected, it is labelled and added to the labelled dataset with a weight set to the inverse of the selection probability (Section~\ref{sec:bias-correction} explains why this is a useful choice).

\begin{figure}[h]
  \hrule
  \vspace{0.8em}

  For each new example $x$:
  \begin{enumerate}
    \item Calculate a selection probability $P_x$ for $x$.
    \item With probability $P_x$: query the label for $x$ and add $x$ to the labelled dataset with importance weight $w_x = \frac{1}{P_x}$.
  \end{enumerate}

  \hrule
  \vspace{0.8em}

  \caption{The importance-weighted active learning algorithm \citep{Beygelzimer2009}.}
  \label{fig:basic-iwal-algorithm}
\end{figure}

Different implementations of importance-weighted active learning have different definitions for the selection probability, but they share the same idea: the probability should be higher if the new example is likely to be interesting, and lower if it is not.
The selection probability should always be greater than zero to ensure that there are no areas in the sample space that will never be sampled.
Because the algorithm explores the full sample space it is guaranteed to converge to the optimal hypothesis, avoiding the missed cluster problem \citep{Dasgupta2008} illustrated in Figure~\ref{fig:missed-cluster-problem}.

The selection probability also determines the size of the training set.
An implementation with higher selection probabilities will select more samples than an implementation that tends to choose smaller selection probabilities.

\begin{figure}[h]
  \begin{center}
    \begin{tikzpicture}[scale=0.5\textwidth/15cm, line width=0.3, >=stealth]
      \fill[color=gray]
            (-4.5,0) -- (-2.5,0) -- (-2.5,1) -- (-4.5,1) -- (-4.5,0)
            (1.5,0) -- (0,0) -- (0,1) -- (1.5,1) -- (1.5,0)
            (7.5,0) -- (5.5,0) -- (5.5,1) -- (7.5,1) -- (7.5,0);

      \draw (-7.5,0) -- (7.5,0)
            (-7.5,0) -- (-5.5,0) -- (-5.5,1) -- (-7.5,1) -- (-7.5,0)
            (-4.5,0) -- (-2.5,0) -- (-2.5,1) -- (-4.5,1) -- (-4.5,0)
            (-1.5,0) -- (0,0) -- (0,1) -- (-1.5,1) -- (-1.5,0)
            (1.5,0) -- (0,0) -- (0,1) -- (1.5,1) -- (1.5,0)
            (7.5,0) -- (5.5,0) -- (5.5,1) -- (7.5,1) -- (7.5,0);

      \draw[->] (-5,2.2) -- (-5,1.2);
      \draw[->] (0,2.2) -- (0,1.2);

      \draw (-5.6,2.2) node[anchor=south west] {$w^*$}
            (-0.6,2.2) node[anchor=south west] {$w$};

      \draw (-6.5,-0.05) node[anchor=north] {45\%}
            (-3.5,-0.05) node[anchor=north] {5\%}
            (   0,-0.05) node[anchor=north] {5\%}
            ( 6.5,-0.05) node[anchor=north] {45\%};
    \end{tikzpicture}
  \end{center}
  \caption{
    The missed cluster problem:
    if the initial samples are drawn from the two groups in the middle, close-to-the-boundary sampling will keep querying samples near the initial boundary $w$.
    Because it never looks elsewhere, the classifier will never find the optimal boundary $w^*$.
    \citep{Dasgupta2008}
  }
  \label{fig:missed-cluster-problem}
\end{figure}
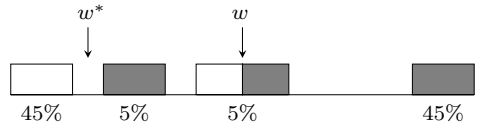

\subsection{Selection probability}
\label{sec:selection-probability}

\citet{Beygelzimer2010} define the selection probability of the $k$th sample as
\begin{align}
  P_k = \min \left\{ 1, \left(\frac{1}{G^2_k} + \frac{1}{G_k}\right) \frac{C_0 \log k}{ k - 1 } \right\} \\
  \text{\hspace{0.5em} where \hspace{0.5em}}
  G_k = \mathrm{err}\left(h'_k, S_k\right) - \mathrm{err}\left(h_k, S_k\right) \nonumber
  \text{.}
\end{align}
It compares the error of two hypotheses $h_k$ and $h'_k$ on the current set of labelled samples $S_k$.
$h_k$ is the current hypothesis, i.e., the hypothesis that minimises the error on $S_k$.
$h'_k$ is the alternative hypothesis: the hypothesis that disagrees with $h_k$ on the label of the sample $k$, but otherwise still minimises the error on $S_k$.
$C_0$ is a free parameter that scales the probabilities, such that choosing a larger $C_0$ will select more samples.

We will explain the intuition behind this definition.
The selection probability is large if the difference in error between the current hypothesis and the alternative hypothesis is small.
The errors of the two hypotheses give an idea of the quality of the current prediction.
By definition, the current hypothesis is optimal on the current set of labelled samples, but it is not necessarily optimal on the true distribution.
The alternative hypothesis cannot be better on the current set of samples, but it might be better on the true distribution.

The error difference predicts how likely it is that the alternative hypothesis is better than the current hypothesis.
If on the true data the alternative hypothesis is better than the current hypothesis, the errors on the current set are likely to be close together.
The current hypothesis can never be too far off because the active learning converges to the optimal solution.
On the other hand, if the current hypothesis is indeed better than the alternative, even on the true data, the error of the alternative hypothesis on the current data might well be higher than that of the current hypothesis.
Thus, the smaller the difference on the current set, the more likely it is that the alternative hypothesis is better.

The error difference also predicts how informative the new example could be.
If the difference is large, the current evidence is strongly in favour of one label and the new sample will probably not add new information.
If the difference is small, the current set of samples does not provide sufficient information to confidently label the new sample, so it might be useful to ask for the label.

\citet{Beygelzimer2010} use this intuition in their definition of the selection probability: a larger difference in error leads to a smaller selection probability.
Note also that the selection probabilities become smaller as the number of samples grows: the more samples seen, the more evidence there already is to support the current hypothesis.
In their paper, \citeauthor{Beygelzimer2010} provide confidence and label complexity bounds for their definition of the probability.
As expected, there is a trade-off between the quality of the prediction and the size of the sample selection.
Large $C_0$ provide large sample selections and results similar to random sampling, whereas small $C_0$ provide smaller sample selections that are sometimes much worse than random selections.

\subsection{Bias correction with importance weights}
\label{sec:bias-correction}

Like any active learning algorithm, importance-weighted active learning creates a biased sample selection.
There are more samples from areas that are interesting to the selector, and fewer from other areas.
Unlike other active learning algorithms, however, importance-weighted active learning uses importance weighting to correct this bias.

Importance-weighted active learning assigns a weight $w_x$ to each example in the selection: larger weights for examples from undersampled areas and smaller weights for the examples that have been oversampled.
Importance-weighted versions of classifiers and performance metrics use these weights in their optimisations.
For example, the importance-weighted zero-one classification error could be defined as the sum of the normalised weights of the misclassified samples.
If the importance weights are set to the correct values, this importance-weighted estimator is an unbiased estimator for the classification error.

The quality of importance weighting depends on the correct choice of the importance weights.
For example, the weights could be derived from density estimates by dividing the density in the true distribution by the density in the sample selection.
However, estimating densities is difficult, so this method can lead to variable and unreliable results.

Importance-weighted active learning has an advantage: because it makes a biased-random sample selection with a known selection probability, the algorithm knows exactly how biased the selection is.
By using the inverse of the selection probability as the importance weight, importance-weighted active learning can calculate the perfect correction to its self-created bias.
\citet{Beygelzimer2010} prove that the importance-weighted estimators of importance-weighted active learning are indeed unbiased.

To see how importance weighting corrects the bias, compare the probability density of an example in random sampling with that same probability density in importance-weighted active learning.
In random sampling, the probability density of $x$ is equal to its density in the true distribution: $P_{RD}\left(x\right) = P\left(x\right)$.
In importance-weighted active learning, the probability density of $x$ in the labelled set depends on two things: the density in the true distribution, i.e., the probability that the sample is offered to the algorithm, and the probability $s\left(x\right)$ that the algorithm decides to label the example: $P_{IWAL}\left(x\right) = P\left(x\right) \cdot s\left(x\right)$.
It gives the example the importance weight $\frac{1}{s\left(x\right)}$, so the expected importance-weighted density of $x$ is the same in both algorithms: $P_{RD}\left(x\right) = \frac{1}{s\left(x\right)} \cdot P_{IWAL}\left(x\right) = P\left(x\right)$.

Although importance weighting corrects the bias and produces unbiased estimators, there are still two important differences between the active selection and the random selection.
First, the importance weights introduce an extra source of variance in the density distribution that may cause suboptimal results.
Second, and more important, the weighted distribution may be correct, but the unweighted distribution is still different.
This means that the active selection has different levels of detail than the random selection.

\subsection{Increased variance caused by large importance weights}
\label{sec:iw-variance}

On average the importance-weighted distribution may be similar to the uniform distribution, but selections in individual runs will be different.
The variance in the distributions from active learning will be higher than in the distributions from random sampling, because of the importance weighting.
The algorithm gives large weights to samples with a small selection probability, to compensate for their relative undersampling in the average dataset.
However, if one of these rare samples is selected, its large importance weight will give it a disproportionally strong influence.
This effect averages out when the number of samples increases, but it could be a problem with small sample sizes.

To show this problem in a practical experiment requires a dataset with outliers that 1.\ cause a significant change in the classifier, 2.\ are so rare and outlying that they are not a problem for random sampling, but 3.\ are still common enough to be picked by the active learning algorithm often enough to be a problem.
Requirements 2 and 3 seem contradictory: the outliers must be rare and frequent at the same time.
One way to achieve these goals is to spread a large number of outliers over a large area.
Individually they are outliers, but there are enough outliers to ensure that one or two will be selected.

Consider this experiment on a dataset with circular outliers (Figure~\ref{fig:iwal-exp-circle}), with a consumer based on quadratic discriminant analysis (QDA) and the linear Vowpal Wabbit\footnote{%
  The Vowpal Wabbit is a system for fast online learning.
  It has a module for importance-weighted active learning.
  The source code is available via \url{http://hunch.net/~vw/} and \url{https://github.com/JohnLangford/vowpal_wabbit/}.
  The experiments in this paper used version 6.1.2 of the Vowpal Wabbit, commit 38734e4f in the Git repository.
} as the selector.
The Vowpal Wabbit implements the importance-weighted active learning algorithm with a selection probability similar to that by \citet{Beygelzimer2010}.
The dataset has two dense clusters in the middle and a circle of outliers orbiting those two clusters.
The outliers have a label that is the opposite of the label of the closest cluster.
By tuning the number of samples in the circle, it is possible to find a distribution where the QDA classifier trained on the active samples performs consistently worse than the QDA classifier trained on random samples (Table~\ref{tab:circle-test-wqda}).
Closer inspection of individual runs with these settings shows that in many cases the QDA classifier is thrown off-balance by an outlier with a heavy weight.

\begin{figure}[t]
  \centering
    \begin{tikzpicture}[scale=0.2\textwidth/15cm, line width=0.3, >=stealth]
      \fill[color=gray]
        (-1,-1) -- (0,-1) -- (0,1) -- (-1,1) -- (-1,-1)
        (0,-9.9) arc(-90:90: 9.9) (0,10.2) arc(90:-90:10.2);

      \draw
        (-1,-1) -- (0,-1) -- (0,1) -- (-1,1) -- (-1,-1) 
        (0,-1) -- (1,-1) -- (1,1) -- (0,1) -- (0,-1);
      \draw
        (0,-9.9) arc(-90:90: 9.9)
        (0,-10.2) arc(-90:90:10.2);
      \draw
        (0,9.9) arc(90:270: 9.9)
        (0,10.2) arc(90:270:10.2);

      \draw (-9.9,0) node[anchor=east] {0.1\%}
            (-1,0) node[anchor=east] {49.9\%}
            ( 1,0) node[anchor=west] {49.9\%}
            ( 9.9,0) node[anchor=west] {0.1\%};
    \end{tikzpicture}
  \caption{The `circle' dataset: a 2D problem with two dense clusters and a very sparse circle with samples from the opposite class.}
  \label{fig:iwal-exp-circle}
\end{figure}
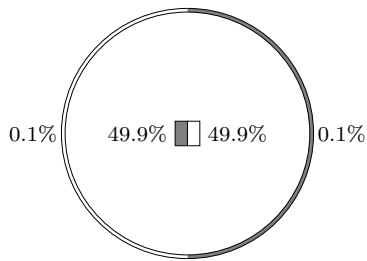

It is quite hard to trigger this behaviour on purpose: it does not happen at small sample sizes, the density of the outliers has to be just right, and this distribution does not cause problems for linear discriminant analysis or linear support vector machines.
Still, this example illustrates that the importance weights can sometimes introduce new problems.

\begin{table}[t]
  \centering
    \begin{tabular*}{\textwidth}{rcc}
      \hline
         Unlabelled samples      & IWAL error                 & Random error \\
      \hline
              10 \hspace{0.8cm}  &   0.15732 $\pm$ 0.00583    &   0.15651 $\pm$ 0.00586   \\
              50 \hspace{0.8cm}  &   0.05967 $\pm$ 0.00215    &   0.06077 $\pm$ 0.00229   \\
             100 \hspace{0.8cm}  &   0.05496 $\pm$ 0.00203    &   0.05203 $\pm$ 0.00211   \\
             500 \hspace{0.8cm}  &   0.04241 $\pm$ 0.00151    &   0.02837 $\pm$ 0.00097   \\
            1000 \hspace{0.8cm}  &   0.03725 $\pm$ 0.00131    &   0.02339 $\pm$ 0.00049   \\
            2500 \hspace{0.8cm}  &   0.03063 $\pm$ 0.00099    &   0.02009 $\pm$ 0.00035   \\
            5000 \hspace{0.8cm}  &   0.02538 $\pm$ 0.00080    &   0.01889 $\pm$ 0.00037   \\
            7500 \hspace{0.8cm}  &   0.02240 $\pm$ 0.00068    &   0.01720 $\pm$ 0.00033   \\
           10000 \hspace{0.8cm}  &   0.02075 $\pm$ 0.00052    &   0.01631 $\pm$ 0.00033   \\
      \hline
    \end{tabular*}
  \caption{Errors of QDA with IWAL and random sampling, on the circle dataset with circle density $0.001$, for different numbers of available (unlabelled) examples. The mean error with IWAL was significantly higher than with random sampling. (Mean $\pm$ std. of the mean.)}
  \label{tab:circle-test-wqda}
\end{table}

\subsection{Differences in the unweighted sampling density}
\label{sec:absolute-density}

The second issue is that even after the importance-weighted correction, the sample selections of importance-weighted active learning are still different from those of random sampling.
In importance-weighted active learning, the unweighted probability density of $x$, $P_{IWAL}\left(x\right)$, depends on the selection probability $s\left(x\right)$ that is determined by the algorithm.
On average, compared with random sampling and the true distribution, the active sample selection will have relatively more examples for which $s\left(x\right)$ is large and relatively fewer examples for which $s\left(x\right)$ is small.
This is normal, since the active learner would not be an active learner if it did not influence the sample selection, but it is a problem if the samples that are excluded are important to train a different consumer.

A simple experiment shows that an importance-weighted active learner has a preference for examples that are interesting to the classifier used to make the selection.
Figure~\ref{fig:plot-simple-line-density-20120705} shows the sample selection of the Vowpal Wabbit on a simple two-class problem with uniformly distributed samples.
Random selection samples with a uniform density, but most of the samples in the active selection are concentrated near the decision boundary at $x=0$.
Importance weighting gives the correct weighted probability density, but the absolute number of selected examples is still different.

\begin{figure}[h]
  \centering
    \includegraphics[width=0.45\textwidth]{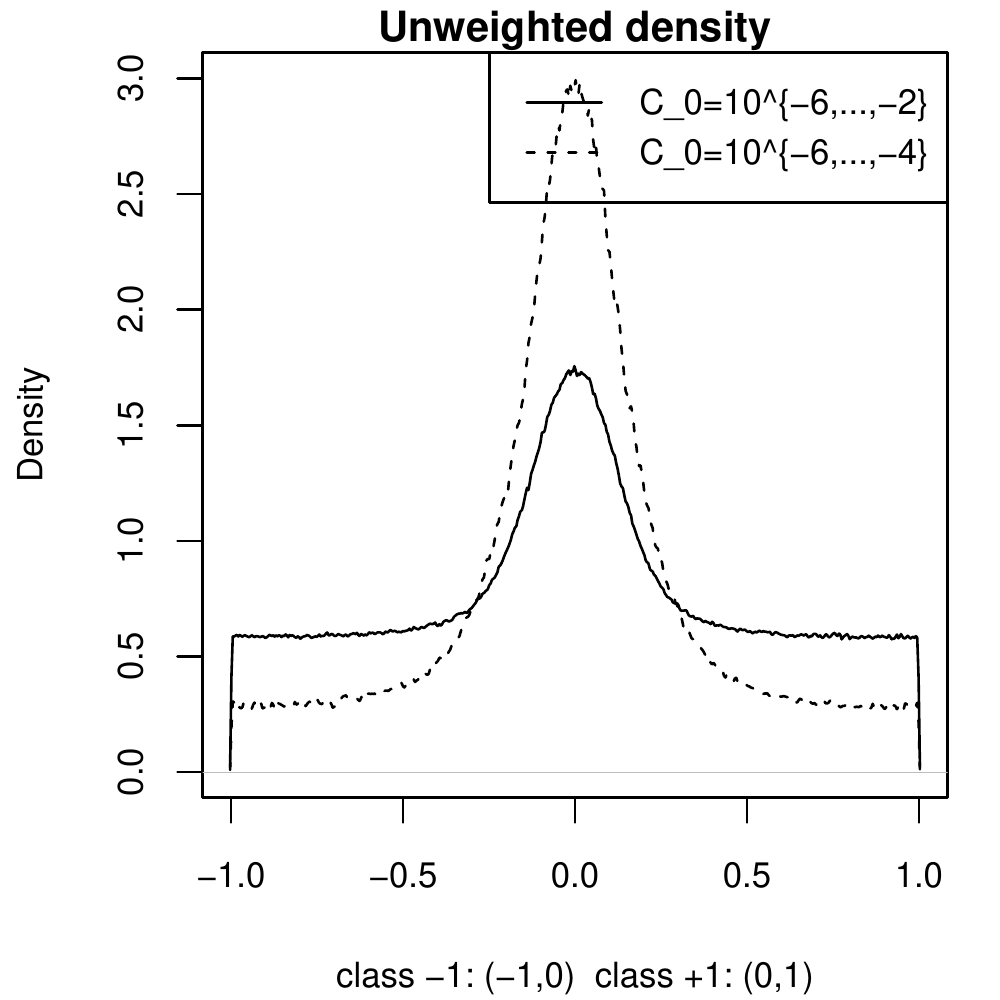}
    \includegraphics[width=0.45\textwidth]{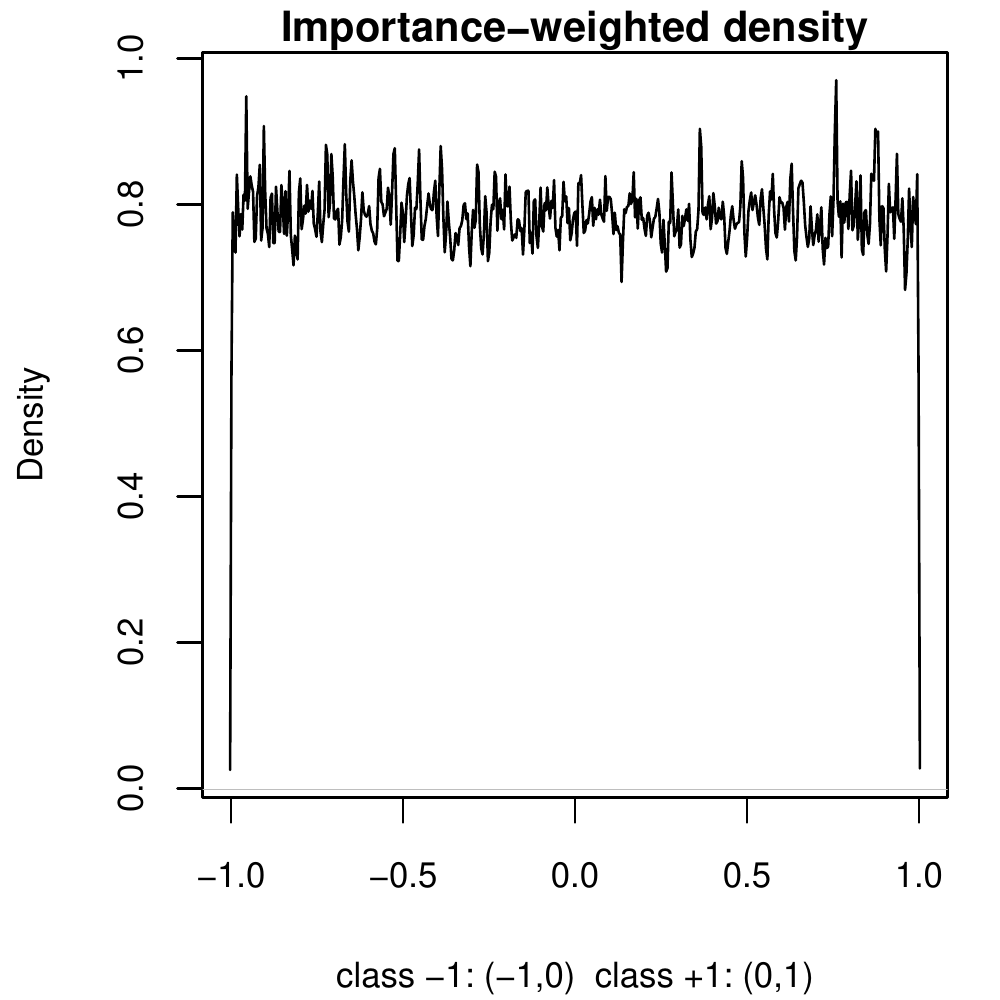}
  \caption{%
    Two plots of the density of active learning selections made by the Vowpal Wabbit (1000 runs with 1000 unlabelled examples, for various $C_0$).
    The 1D problem has two uniform classes, class $-1$ at $x=\left[-1,0\right)$ and class $+1$ at $x=\left[0,1\right]$.
    The unweighted density (left) shows that the algorithm selects most examples from the area around the decision boundary.
    The peak in the middle is more pronounced if the algorithm can select fewer examples (a smaller $C_0$, the dashed line in the plot on the right).
    The importance-weighted density (right) follows the true class distribution.
    Most queries outside that area were made with the less aggressive settings (a larger $C_0$, the solid line in the plot).
  }
  \label{fig:plot-simple-line-density-20120705}
\end{figure}

The unweighted sampling density affects the level of detail in the sample space.
In some areas, the active selection provides more detail than a random selection, but in other areas the active selection provides less detail.
This can be a problem if the underrepresented areas are useful the new classifier.

\section{Sample reusability with importance-weighted active learning}
\label{sec:sample-reusability-in-iwal}

Because it provides perfect importance weights, importance-weighted active learning can perhaps provide a better sample reusability than unweighted active learning strategies.
The skewed, biased density distribution of the active selection can be corrected into an unbiased estimate of the true density distribution.
This removes one component that limits sample reusability, but it is not enough because the sample selections are still different.

The sample selection of importance-weighted active learning has a different distribution of detail than random selections.
This has consequences for the sample reusability.
In the areas that are undersampled by active learning, the random selection can provide more detail than the active selection.
The consumer might depend on information from those areas.
If the active selection provides less detail than the random selection, the consumer might have learned more from the random selection: sample reusability cannot be guaranteed.

Whether this problem occurs in practice depends first of all on the correspondence between the selector and the consumer.
If both classifiers are interested in the same areas of the sample space, the examples will also be useful to both classifiers.
But even if only some parts of the areas of interest overlap, active learning could still be better if the improvement in one area is large enough to compensate for the loss in another area.

The number of samples is also important.
At small sample sizes the lack of detail can matter.
At larger sample sizes the difference becomes less noticeable -- there are more samples even in undersampled areas -- but the active selection may still have less detail than a random sample of equal size.

Since there are almost always some consumers that need details that the selector did not provide, the conclusion must be that sample reusability cannot be guaranteed.
The lack of detail in undersampled areas means that there is always the possibility that a consumer would have learned more from a random than from an active sample selection.

\section{Demonstrating a sample reusability problem}
\label{sec:experiment-synthetic}

It is possible to construct an experiment that shows that importance-weighted active learning does not always provide sample reusability.
Choose a selector and a consumer, such that the consumer can produce all of the solutions of the consumer, plus some solutions that the selector cannot represent.
More formally: the hypothesis space of the selector should be a proper subset of the hypothesis space of the consumer.
Next, construct a classification problem where the optimal solution for the selector is different from the optimal solution of the consumer.
Run the experiment: use importance-weighted active learning to select samples, train a consumer on the active selection and a consumer on a random selection of the same size. Measure the errors on a held-out test set.
For small sample sizes, expect consumer trained on a random sample selection to perform better than the consumer trained on the active selection.

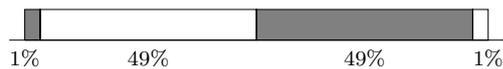
\begin{figure}[h]
  \centering%
    \begin{tikzpicture}[scale=0.5\textwidth/15cm, line width=0.3, >=stealth]
      \fill[color=gray]
            (-7.5,0) -- (-7,0) -- (-7,1) -- (-7.5,1) -- (-7.5,0)
            (0,0) -- (7,0) -- (7,1) -- (0,1) -- (0,0);

      \draw (-8,0) -- (8,0)
            (-7.5,0) -- (-7,0) -- (-7,1) -- (-7.5,1) -- (-7.5,0)
            (-7,0) -- (0,0) -- (0,1) -- (-7,1) -- (-7,0)
            (0,0) -- (7,0) -- (7,1) -- (0,1) -- (0,0)
            (7,0) -- (7.5,0) -- (7.5,1) -- (7,1) -- (7,0);

      \draw (-7.5,-0.05) node[anchor=north] {1\%}
            (-3.5,-0.05) node[anchor=north] {49\%}
            ( 3.5,-0.05) node[anchor=north] {49\%}
            ( 7.5,-0.05) node[anchor=north] {1\%};
    \end{tikzpicture}
  \caption{A simple 1D dataset. Most of the samples are in the middle clusters, each of the clusters on the side has $1\%$ of the samples.}
  \label{fig:detailed-line-diagram}
\end{figure}

For example, consider a one-dimensional two-class dataset (Figure~\ref{fig:detailed-line-diagram}) with four clusters in a $\mathord{+}\,\mathord{-}\,\mathord{+}\,\mathord{-}$ pattern: a tiny cluster of the $+$ class, a large cluster of $-$, an equally large cluster of $+$ and a tiny cluster of $-$.
The optimal decision boundary for a linear classifier is in the middle between the two large clusters.
It will misclassify the two tiny clusters, but that is inevitable.
A more complex classifier can improve on the performance of the simple classifier if it assigns those clusters to the correct class.

Active learning with the linear classifier as the selector selects most samples from the area near the middle, since that is where the linear classifier can be improved.
As a result, the active sample selection does not include many examples from the tiny clusters.
This makes it harder for the complex classifier to learn the correct classification for those areas.
The random selection includes more samples from the tiny clusters, which should give the complex classifier a better chance to learn the correct labels for those clusters.
In this example the expected performance with random sampling is better than with active learning: the sample selection from active learning is not reusable.

Plots of the learning curve (Figure~\ref{fig:plot-detailed-line-test-20120705}) show this effect in an experiment with the linear Vowpal Wabbit as the selector and a support vector machine with a radial basis function kernel as the consumer.
The plots show the results of experiments with different values for the parameter $C_0$, used in the definition of the selection probability in the Vowpal Wabbit.
In the definition of the selection probability, $C_0$ determines the aggressiveness of the active learner.
If $C_0$ is very small, the selection probabilities will be smaller and the number of selected samples will be small, too.
For larger $C_0$, the selection probability increases and so does the number of selected samples.
At some point the selection probabilities are so large that the active learner will select every sample.
These extreme cases are not representative for real active learning problems.

As predicted, the selection by the simple linear selector is of limited use to the more complex radial-basis support vector machine.
This combination of classifiers gives results that are worse than those of random sampling.
The effect is at its strongest at small sample sizes, but is still noticeable at larger the sample sizes.
The number of labelled samples can increase for two reasons: because the number of unlabelled examples is larger or because $C_0$ is higher.
In both cases the number of samples from undersampled areas increases and the consumer will receive more information on the samples in the edges.
Beyond this crossover point the active learner could perform better than random sampling, even if it undersamples some important areas, because the extra efficiency in the other areas more than compensates for the loss of precision.

\begin{figure}[ht]
  \centering
    \includegraphics[width=\textwidth]{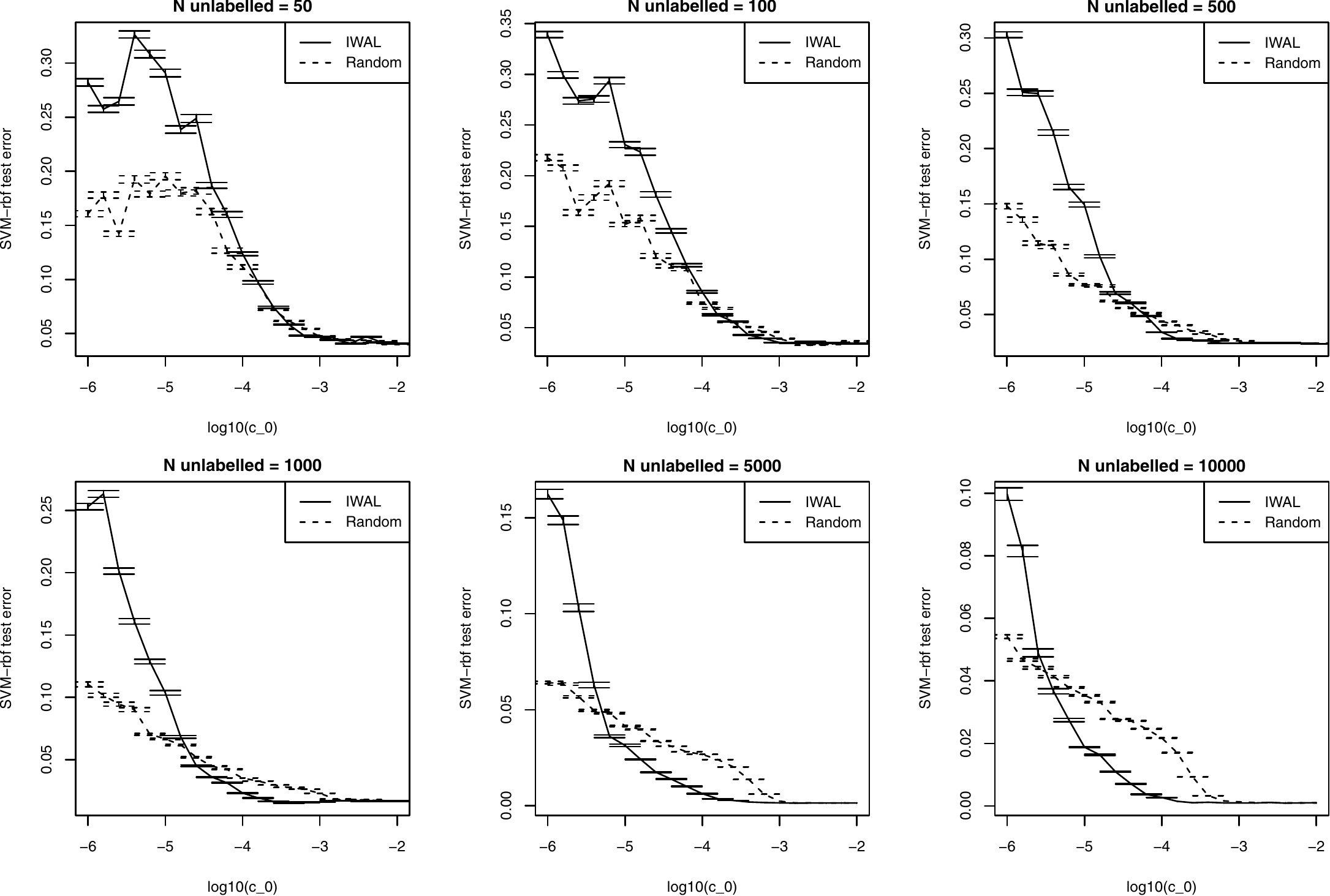}
  \caption{%
    The mean test error of a radial basis support vector machine classifier trained on samples selected by the Vowpal Wabbit importance-weighted active learning (IWAL) algorithm.
    The dashed lines show the same performance of a classifier trained on a random sample selection.
    The error bars indicate the standard deviation of the means.
    The dataset is the dataset shown in Figure~\ref{fig:detailed-line-diagram}.
  }
  \label{fig:plot-detailed-line-test-20120705}
\end{figure}

\section{Experiments}
\label{sec:practice}

The previous experiment shows that sample reusability problems can be triggered using a specially-designed dataset.
This section presents experiments on existing datasets from the UCI Machine Learning Repository and the Active Learning Challenge 2010, comparing the results of three selection strategies: random sampling, uncertainty sampling and importance-weighted active learning.
To evaluate the contribution of the importance weights, we also include importance-weighted active learning without using the importance weights.

\subsection{Datasets}

Two datasets are part of the UCI Machine Learning Repository \citep{UCImlrepo}: \texttt{car} and \texttt{mushroom}.
Two other datasets come from the Active Learning Challenge 2010\footnote{%
  \raggedright
  The datasets from the Active Learning Challenge 2010 are available from \url{http://www.causality.inf.ethz.ch/activelearning.php} and \url{http://jmlr.org/proceedings/papers/v16/}.
} \citep{Guyon2010,Guyon2011}: \texttt{alex}, a dataset generated by a Bayesian network model for lung cancer, and \texttt{ibn\_sina}, from an Arabic handwriting recognition problem.
Table~\ref{tab:uci-datasets} shows some properties of these datasets.
All datasets are two-class classification problems.

\begin{table*}[t]
  \begin{tabular*}{\textwidth}{llccc}
    \hline
    \hspace{1pt}Dataset    & Features                   & Examples   & Positive  & Test prop. \\
    \hline
    \multicolumn{5}{l}{\textit{UCI Machine Learning Repository}} \\
    \hspace{6pt}car        & 6, categorical             & 1728       & 30.0\%    & 10\%             \\
    \hspace{6pt}mushroom   & 20, categorical            & 8124       & 51.8\%    & 20\%             \\
    \multicolumn{5}{l}{\textit{Active Learning Challenge 2010}} \\
    \hspace{6pt}alex       & 11, binary                 & 10000      & 73.0\%    & 20\%             \\
    \hspace{6pt}ibn\_sina  & 92, binary, numerical      & 20722      & 37.8\%    & 20\%             \\
    \hline
  \end{tabular*}
  \caption{Some statistics of the datasets used in these experiments.}
  \label{tab:uci-datasets}
\end{table*}

\subsection{Sample selection strategies}

We used a customised version of the Vowpal Wabbit to perform the sample selection.
The Vowpal Wabbit comes with importance-weighted active learning, to which we added the strategies for uncertainty sampling and random sampling.
The fourth selection strategy, listed as `IWAL (no weights)' in the graphs, uses the sample selection from importance-weighted active learning but sets the importance weight of every example to $1$.
The Vowpal Wabbit creates linear classifiers, so the selector in these experiments is always a linear classifier.

\subsection{Classifiers}

For the consumers, we used six importance-weighted classifiers from R \citep{R2012}, most of them from the \texttt{locClass} package \citep{locClass2010}.
Linear regression (\texttt{lm}) is the most similar model to the linear model in the Vowpal Wabbit.
Linear discriminant analysis (\texttt{lda}) is also a linear classifier, but based on different principles.
We also used quadratic discriminant analysis (\texttt{qda}).
We used support vector machines with a linear kernel (\texttt{wsvm-linear}), a third-degree polynomial kernel (\texttt{wsvm-poly3}), and a radial-basis kernel (\texttt{wsvm-radial}).

For practical reasons, we had to remove some experiments from the results.
Occasionally, especially at small sample sizes, the sample selection strategy selected examples from only one class, making it impossible to train a classifier.
On some datasets the LDA and QDA classifiers complained about singular data.

\subsection{Implementation}

We split the dataset in training set and a test set.
For random selection (Random) we shuffled the training set and selected the first $n$ samples, for $n$ at fixed intervals from only a few samples to every available sample.
For uncertainty selection (US) we used the Vowpal Wabbit to rank the examples and selected the first $n$ samples.
For importance-weighted active learning (IWAL) we used the Vowpal Wabbit with multiple $C_0$ values from $10^{-9}$ to $10^{-2}$.
For the datasets in these experiments this range of $C_0$ produces sample selections with only a few samples, selections that include every available sample and everything in between.
For IWAL without weights we used the IWAL selection but set the weights to a constant $1$.
We report the mean classification accuracy averaged over 100 random train/test splits.

\subsection{Results}

The plots show a selection of the learning curves for these experiments.
The lines show the mean classification error on the test set, the semi-transparent bands indicate the standard deviation of the mean.

For random sampling and uncertainty sampling these plots are straightforward: they show the mean for each value of $n$.
The calculations for importance-weighted active learning are more complicated.
In importance-weighted active learning the number of samples is a random variable.
In our experiments there was seldom more than one value at a specific sample size.
The graphs show the mean and standard deviation for each group of results sharing the same $C_0$.
These group means are plotted at the median sample size for that $C_0$.

\begin{figure}[t]
  \includegraphics[width=0.49\linewidth]{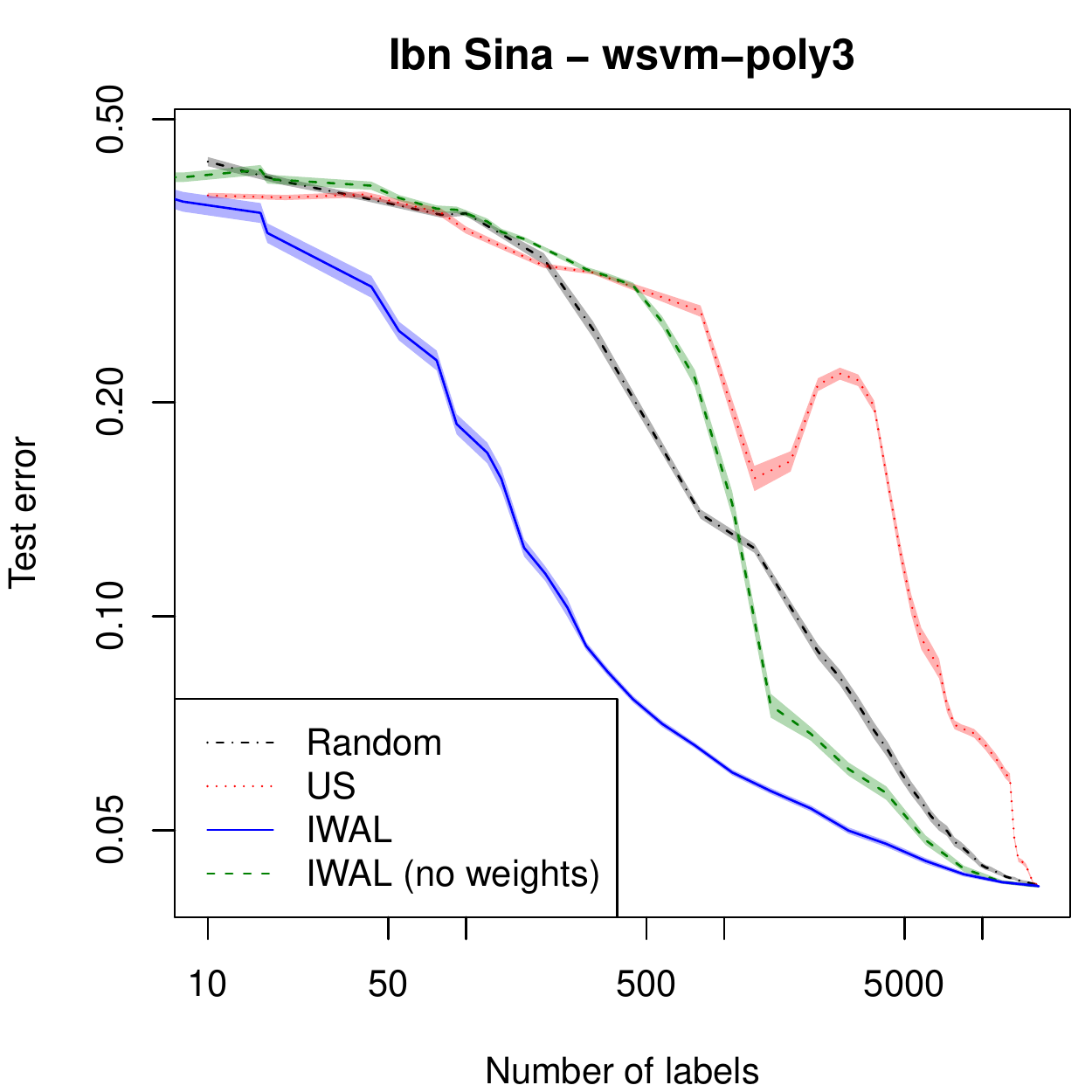}
  \includegraphics[width=0.49\linewidth]{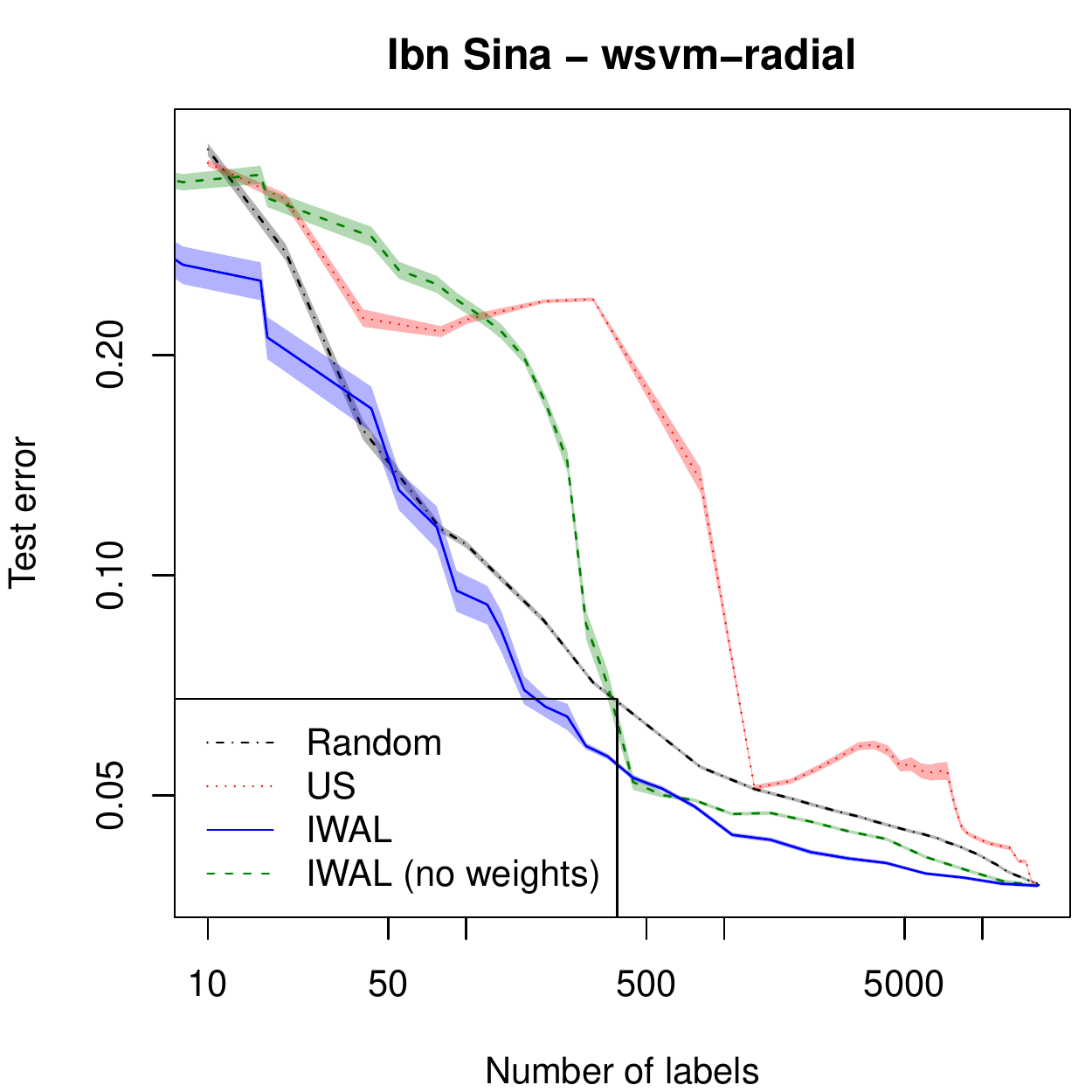}
  \caption{%
    Importance-weighted active learning on the Ibn Sina dataset produced a selection that was reusable by polynomial and radial-basis support vector machines.
  }
  \label{fig:positive-results-ibn-sina}
\end{figure}

When interpreting these graphs, be aware that the results at the maximum sample sizes may not representative for the behaviour of the algorithms on real data.
In these experiments there is a limited amount of data, so it is possible to select and label every example.
As a result, most learning curves end in the same point: there is no difference between learning strategies if they select every sample.
In practice there would be more samples than could ever be labelled, so it is best to look at the middle section of each graph.

In some experiments importance-weighted active learning clearly produced reusable selections.
For example, the learning curves on the Ibn Sina dataset (Figure~\ref{fig:positive-results-ibn-sina}) show good results of with the polynomial and radial-basis support vector machines.
The polynomial kernel achieved a higher accuracy with the selection from importance-weighted active learning than with the random selection.
The radial-basis kernel did not perform better with importance-weighted active learning, but also not much worse.
Uncertainty sampling, however, produced results that were much worse than random sampling.
It is hard to say why these classifiers show these results: apparently, on this dataset, they are interested in similar examples.

\begin{figure}[t]
  \includegraphics[width=0.49\linewidth]{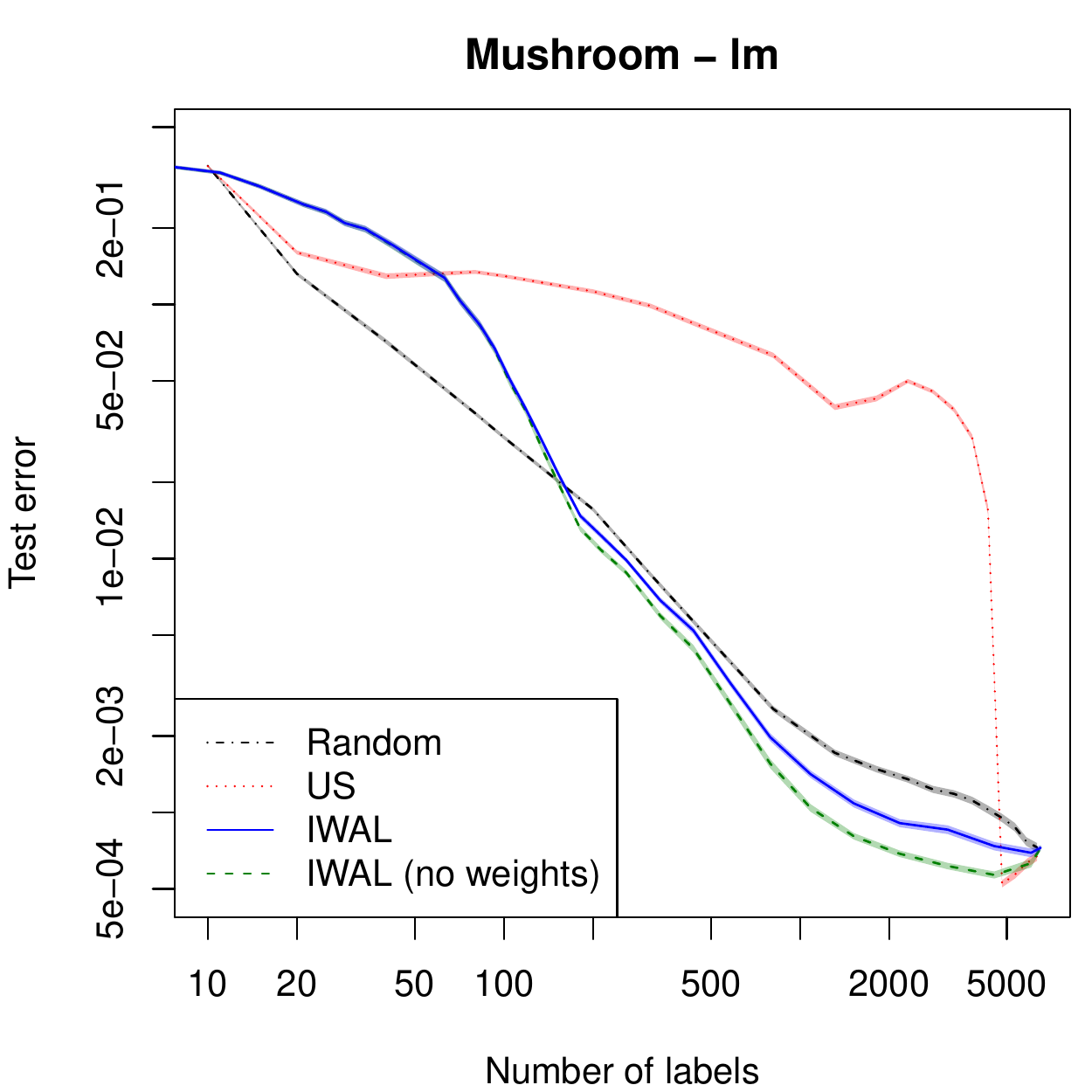}
  \includegraphics[width=0.49\linewidth]{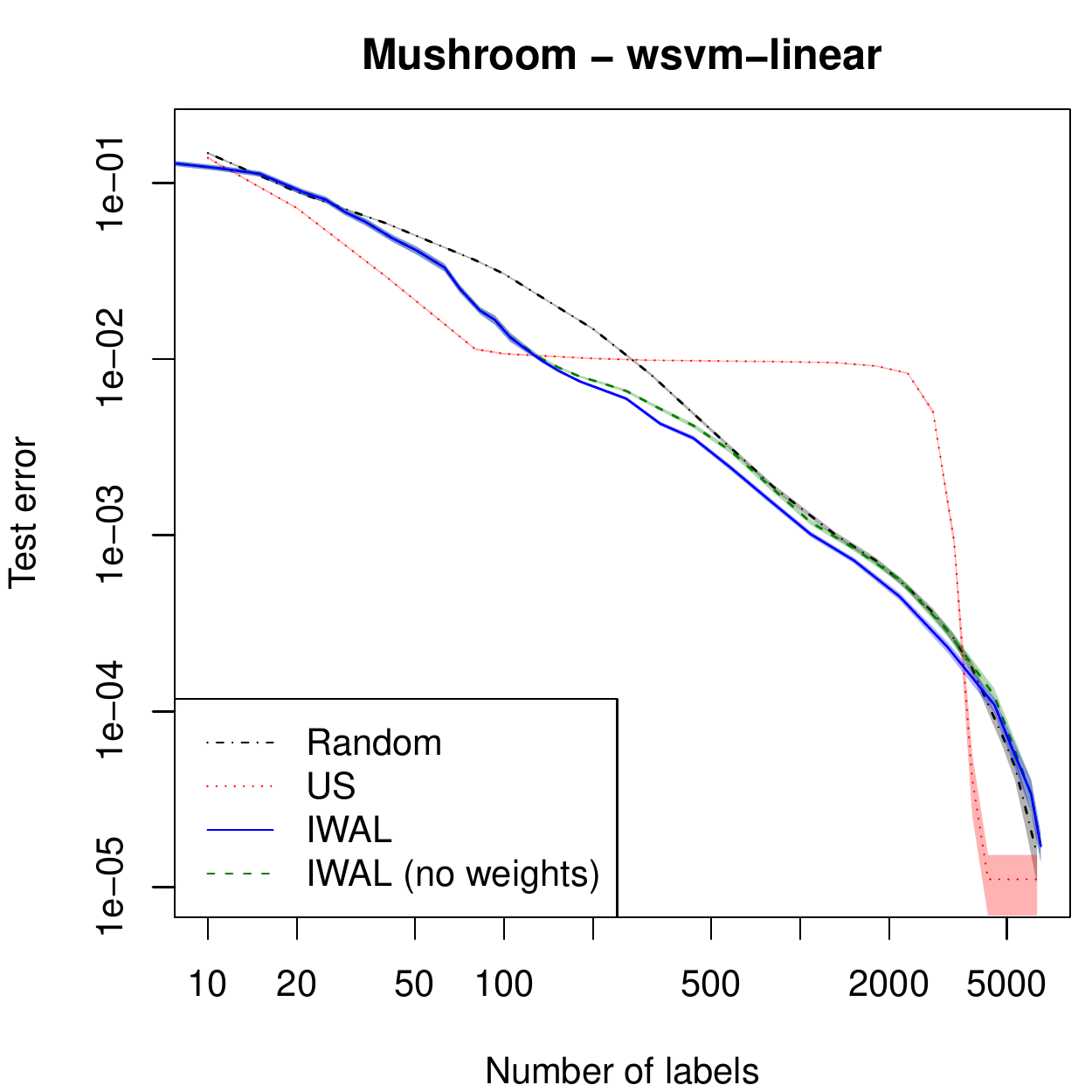}
  \caption{%
    On the mushroom dataset, importance-weighted active learning had an advantage over uncertainty sampling, which may have found a local optimum.
    Similar behaviour can be seen in other datasets.
  }
  \label{fig:missing-cluster-mushroom}
\end{figure}

Sometimes the results of importance-weighted active learning were close to those of random sampling, while uncertainty sampling performed much worse (see, for example, the mushroom dataset in Figure~\ref{fig:missing-cluster-mushroom}).
When this happens uncertainty sampling may have found a local optimum by completely disregarding some areas in the sampling space.
This is a known problem of uncertainy sampling, which may miss some clusters in the data \citep{Dasgupta2008}.
Importance-weighted active learning, where the selection probability is always larger than zero, is less likely to miss these clusters.
This makes active learning safer: the results for importance-weighted active learning might not have been better than random sampling, but they were also not much worse.

However, for some datasets importance-weighted active learning did perform worse than the other sampling methods (e.g., Figure~\ref{fig:importance-weighting-bad}).
This shows that the sample selection by importance-weighted active learning is not always reusable.

\begin{figure}[t]
  \includegraphics[width=0.49\linewidth]{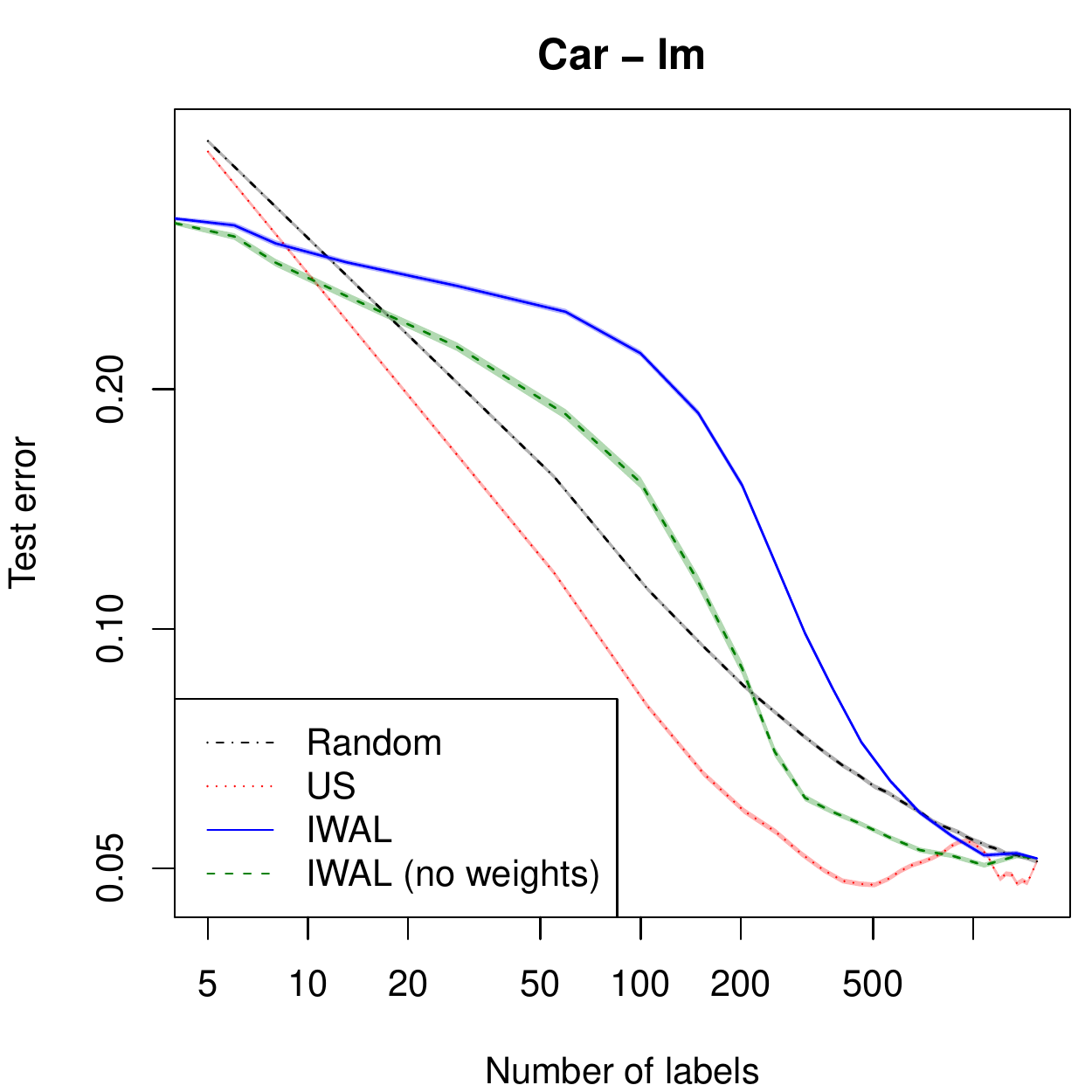}
  \includegraphics[width=0.49\linewidth]{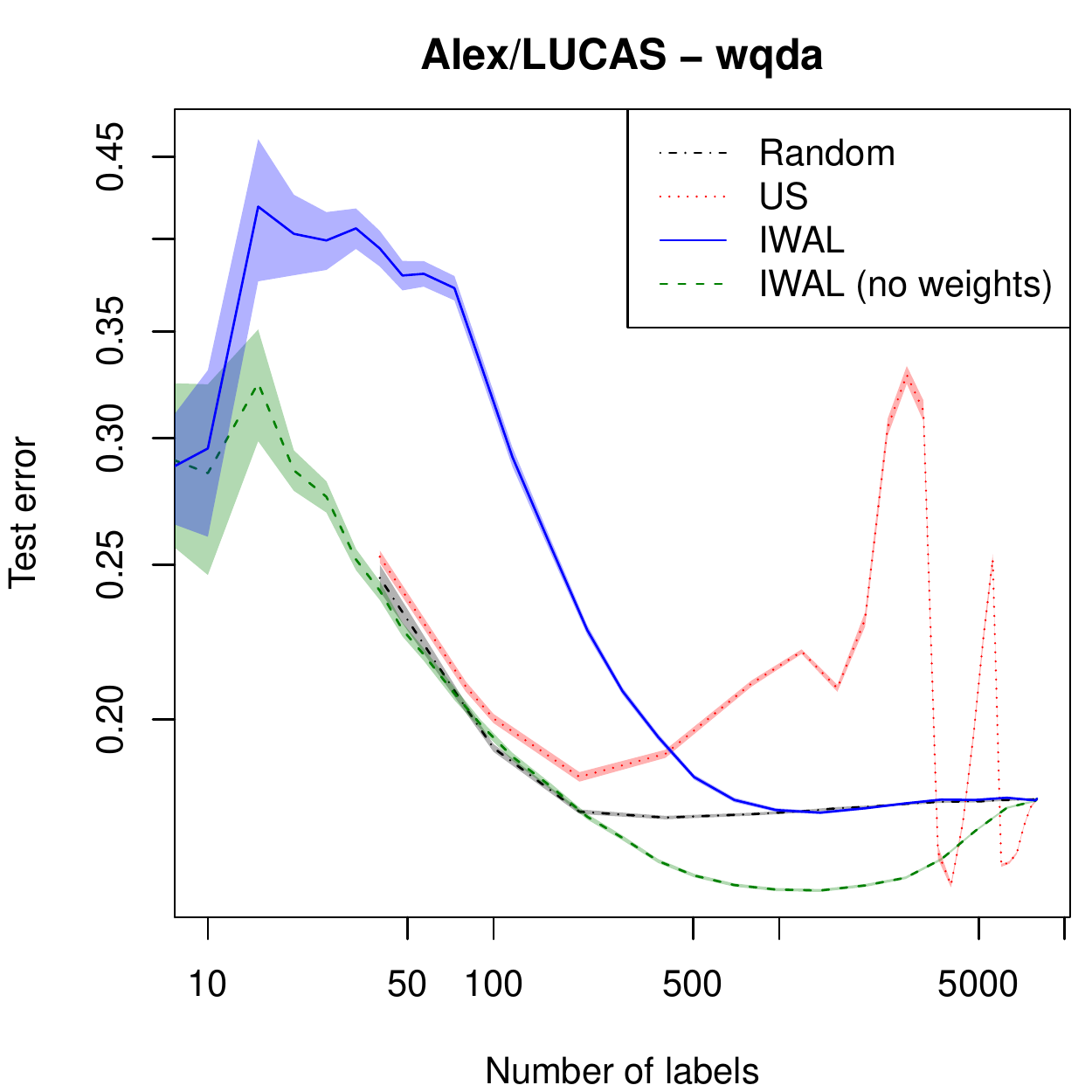}
  \caption{%
    Sometimes importance weights can be the problem: removing the weights from the importance-weighted selection improves the results.
    This happens most often with LDA/QDA consumers and on the \texttt{alex} dataset and could be due to the instability that is introduced by large importance weights.
  }
  \label{fig:importance-weighting-bad}
\end{figure}

In quite a few cases, removing the importance weights improved the performance (e.g., Figure~\ref{fig:importance-weighting-bad}).
For example, this happened with LDA classifiers and on the \texttt{alex} dataset.
Perhaps this is a result of the variability from the large importance weights.
It is curious to see that when there is no reusability for importance-weighted active learning, the unweighted selection often \textit{does} show reusability.
Unfortunately, simply removing the importance weights is not a solution: there are also examples where importance-weighted active learning does work but the unweighted version does not (e.g., Figure~\ref{fig:positive-results-ibn-sina}).

From these results it becomes clear that importance-weighted active learning is only a partial solution for the reusability problem.
There are certainly cases where the samples are reusable, and perhaps they are reusable in more cases than with uncertainty sampling.
But there are also examples where the selection from importance-weighted active learning is not reusable, which suggests that it does not guarantee reusability in every situation.

\section{Conditions for reusability}
\label{sec:conditions}

The theoretical discussion and the experiments show that there is no sample reusability between \textit{every} pair of classifiers on \textit{every} possible problem, but also that there are combinations where sample reusability does appear.
When can a selector-consumer pair guarantee good results?
This section explores possible conditions for sample reusability.

\subsection{Hypothesis spaces and hypotheses}
Let $\mathcal{H}_{sel}$ and $\mathcal{H}_{cons}$ be the hypothesis spaces of the selector and consumer, respectively.
The hypothesis space is the set of all hypotheses that a classifier algorithm can form: for example, the hypothesis space of a linear classifier on a two-dimensional problem is the set of all lines.
Let $\mathrm{err}\left(h, S\right)$ denote the (negative) performance of a hypothesis $h$ on a sample selection $S$, where a smaller $\mathrm{err}\left(h, S\right)$ is better.
If $S$ is an importance-weighted sample selection with weights $W$, $\mathrm{err}\left(h, S, W\right)$ is the importance-weighted error.
Let $\mathrm{err}\left(h\right)$ be the expected performance of hypothesis $h$ on new, unseen data from the true distribution.

\begin{definition}[Optimal hypotheses]
  Define the optimal hypothesis for the selector, $h^*_{sel}$, and the optimal hypothesis for the consumer, $h^*_{cons}$, as the hypotheses that minimise the expected error on unseen data:
  \begin{align}
    h^*_{sel} &= \arg\min \left\{ \mathrm{err}\left(h\right) : h \in \mathcal{H}_{sel} \right\}
    \text{,} \\
    h^*_{cons} &= \arg\min \left\{ \mathrm{err}\left(h\right) : h \in \mathcal{H}_{cons} \right\}
    \text{.}
  \end{align}
\end{definition}

Let $S_{AL,n}$ and $S_{RD,n}$ be the sample selections of $n$ labelled samples, made with active learning (AL) and random sampling (RD), respectively.
For importance-weighted active learning, $W_{AL,n}$ represents the importance weights assigned to the samples in $S_{AL,n}$.

\begin{definition}[Optimal hypotheses on the sample selections]
  Define the optimal hypotheses for the selector and consumer on the importance-weighted active sample selection and on the random sample selection as
  \begin{align}
    h_{sel,AL,n} &= \arg\min \left\{ \mathrm{err}\left(h, S_{AL,n}, W_{AL,n}\right) : h \in \mathcal{H}_{sel} \right\} \text{,} \\
    h_{sel,RD,n} &= \arg\min \left\{ \mathrm{err}\left(h, S_{RD,n}\right) : h \in \mathcal{H}_{sel} \right\} \text{,} \\
    h_{cons,AL,n} &= \arg\min \left\{ \mathrm{err}\left(h, S_{AL,n}, W_{AL,n}\right) : h \in \mathcal{H}_{cons} \right\} \text{,} \\
    h_{cons,RD,n} &= \arg\min \left\{ \mathrm{err}\left(h, S_{RD,n}\right) : h \in \mathcal{H}_{cons} \right\} \text{.}
  \end{align}
\end{definition}

\begin{assumption}[Functional classifiers]
  Assume that the classifier indeed minimises $\mathrm{err}\left(h,S\right)$, that is, that it minimises the empirical risk and does indeed select these hypotheses when given $S_{AL,n}$ or $S_{RD,n}$.
\end{assumption}

\subsection{Expected error and sample reusability}
Before defining any conditions for reusability, we should take a closer look at the hypotheses that are selected with active learning and with random sampling.
More specific, we should derive the expected error of these hypotheses, both on the active and random sample selections and on unseen data.
Sample reusability is defined in terms of these expected errors.

First, note that random sampling gives unbiased estimates of the error: $\mathrm{E}\left[\mathrm{err}\left(h,S_{RD,n}\right)\right] = \mathrm{err}\left(h\right)$, the expected error of a hypothesis $h$ on a random sample selection is equal to the expected error on unseen data.
Or: the empirical risk averaged over all possible random sample selections is equal to the true error.
\begin{lemma}
  The optimal hypothesis $h^*_{cons}$ -- the hypothesis with the smallest expected error on unseen data -- will also have the smallest expected error on the random sample selection:
  \begin{align}
    \mathrm{E}\left[\mathrm{err}\left(h^*_{cons},S_{RD,n}\right)\right]
      &= \min \left\{ \mathrm{E}\left[\mathrm{err}\left(h, S_{RD,n}\right)\right] : h \in \mathcal{H}_{cons}\right\} \text{.}
  \end{align}
\end{lemma}
\begin{proof}
  \begin{align*}
    \mathrm{E}\left[\mathrm{err}\left(h^*_{cons},S_{RD,n}\right)\right]
      &= \mathrm{err}\left(h^*_{cons}\right) \\
      &= \min \left\{ \mathrm{err}\left(h\right) : h \in \mathcal{H}_{cons}\right\} \\
      &= \min \left\{ \mathrm{E}\left[\mathrm{err}\left(h, S_{RD,n}\right)\right] : h \in \mathcal{H}_{cons}\right\} \text{.}
  \end{align*}
\end{proof}

Importance-weighted active learning has the same property.
Because it uses importance weighting, importance-weighted active learning produces an unbiased estimator of the error, i.e., $\mathrm{E}\left[\mathrm{err}\left(h,S_{AL,n}, W_{AL,n}\right)\right] = \mathrm{err}\left(h\right)$ \citep{Beygelzimer2009}.
The average $\mathrm{err}\left(h,S_{AL,n},W_{AL,n}\right)$ over all importance-weighted sample selections $S_{AL,n}$ is equal to the expected error on unseen data.
\begin{lemma}
  The expected error of the optimal hypothesis $h^*_{cons}$ will also be optimal on the importance-weighted sample:
  \begin{align*}
    \mathrm{E}\left[\mathrm{err}\left(h^*_{cons},S_{AL,n},W_{AL,n}\right)\right]
      &= \min \left\{ \mathrm{E}\left[\mathrm{err}\left(h, S_{AL,n},W_{AL,n}\right)\right] : h \in \mathcal{H}_{cons}\right\}
    \text{.}
  \end{align*}
\end{lemma}
\begin{proof}
  \begin{align*}
    \mathrm{E}\left[\mathrm{err}\left(h^*_{cons},S_{AL,n},W_{AL,n}\right)\right]
      &= \mathrm{err}\left(h^*_{cons}\right) \\
      &= \min \left\{ \mathrm{err}\left(h\right) : h \in \mathcal{H}_{cons}\right\} \\
      &= \min \left\{ \mathrm{E}\left[\mathrm{err}\left(h, S_{AL,n},W_{AL,n}\right)\right] : h \in \mathcal{H}_{cons}\right\}
    \text{.}
  \end{align*}
\end{proof}

In both cases, with importance-weighted active learning and random sampling, the optimal hypothesis has the lowest expected error on the sample selection.
This does not mean that the optimal hypothesis will also be selected.
The hypotheses $h_{cons,RD,n}$ and $h_{cons,AL,n}$ are the hypotheses with the lowest error on the sample selection.
The optimal hypothesis has the best \textit{expected} error and it will have the smallest error if the sample size is unlimited, but in an individual sample selection with a limited size there may be another hypothesis that has a smaller error on the training set.
This difference is relevant for sample reusability.
Since the selected hypotheses $h_{cons,RD,n}$ and $h_{cons,AL,n}$ are not optimal, they will have a larger expected error than the optimal hypothesis $h^*_{cons}$.
Similarly, the hypotheses selected by the selector, $h_{sel,RD,n}$ and $h_{sel,AL,n}$ will also have an expected error that is higher than $h^*_{sel}$.

\begin{definition}[Additional error]
  Let $\varepsilon_{sel,RD,n}$, $\varepsilon_{sel,AL,n}$, $\varepsilon_{cons,RD,n}$ and $\varepsilon_{cons,AL,n}$ represent the extra error introduced by the selection strategy and the sample size of random sampling and active learning, such that
  \begin{align*}
    \mathrm{err}\left(h_{sel,RD,n}\right) &= \mathrm{err}\left(h^*_{sel}\right) + \varepsilon_{sel,RD,n} \\
    \mathrm{err}\left(h_{sel,AL,n}\right) &= \mathrm{err}\left(h^*_{sel}\right) + \varepsilon_{sel,AL,n} \\
    \mathrm{err}\left(h_{cons,RD,n}\right) &= \mathrm{err}\left(h^*_{cons}\right) + \varepsilon_{cons,RD,n} \\
    \mathrm{err}\left(h_{cons,AL,n}\right) &= \mathrm{err}\left(h^*_{cons}\right) + \varepsilon_{cons,AL,n} \\
  \end{align*}
\end{definition}

\begin{assumption}[Functional active learner]
  Assume that the active learner works, i.e., that it produces hypotheses that are better than random sampling, at least for the selector.
  The active hypothesis $h_{sel,AL,n}$ is expected to have a lower error than the random hypothesis $h_{sel,RD,n}$ for an equal sample size:
  \[
    \mathrm{err}\left(h_{sel,AL,n}\right) \le \mathrm{err}\left(h_{sel,RD,n}\right)
    \text{,}
  \]
  which implies that the extra error $\varepsilon_{sel,AL,n} \le \varepsilon_{sel,RD,n}$.
\end{assumption}

\begin{definition}[Sample reusability]
  There is sample reusability with a consumer if the hypothesis produced by active learning is not worse than the hypothesis produced by random sampling:
  \[
    \text{sample reusability if }
    \mathrm{err}\left(h_{cons,AL,n}\right) \le \mathrm{err}\left(h_{cons,RD,n}\right)
    \text{.}
  \]
\end{definition}

We assumed that the active learner is functional, but does that also imply that there is sample reusability?
The rest of this section tries to answer this question.
We try to formulate conditions that guarantee that there is sample reusability between a pair of classifiers.
Note that the conditions should \textit{guarantee} reusability: it is not enough to have reusability in most problems, the conditions should be such that, for classifier pairs that meet them, there is reusability on \textit{all} problems.
In other words: we require that $\mathrm{err}\left(h_{cons,AL,n}\right) \le \mathrm{err}\left(h_{cons,RD,n}\right)$ on \textit{all} problems.

\subsection{Necessary condition 1: $\mathcal{H}_{cons} \subseteq \mathcal{H}_{sel}$ }

The first condition that is necessary to guarantee sample reusability is that the hypothesis space of the consumer is a subset of the hypothesis space of the selector.
Suppose that the hypothesis space of the consumer is \textit{not} a subset of the hypothesis space of the selector, so that $\mathcal{H}_{cons} \backslash \mathcal{H}_{sel} \neq \varnothing$.

The active selection is focused on the best hypothesis in $\mathcal{H}_{sel}$, to approximate $h^*_{sel}$ and to make sure that $\mathrm{err}\left(h^*_{sel}, S_{AL,n},W_{AL,n}\right)$ is better than the $\mathrm{err}\left(h, S_{AL,n},W_{AL,n}\right)$ of every other hypothesis $h \in \mathcal{H}_{sel}$.
With a limited number of samples, more focus on $\mathrm{err}\left(h^*_{sel}\right)$ means less focus on other hypotheses: the active learner undersamples some areas of the sample space.
The hypotheses in $\mathcal{H}_{cons} \backslash \mathcal{H}_{sel}$ are completely new.
The active sample selection may include some information about these hypotheses, but that is uncertain and there will be problems where little or no information is available.
The random sample selection, on the other hand, did not focus on $\mathcal{H}_{sel}$ and therefore did not have to focus less on $\mathcal{H}_{cons}$: it did not undersample any area.
In those cases it is likely that the random sample selection provides more information about $\mathcal{H}_{cons}$ than the active sample selection.

\begin{figure}[ht]
  \centering
    \begin{minipage}[t]{0.47\textwidth}
      \centering
        \begin{tikzpicture}[scale=1, line width=0.3, >=stealth]
          \begin{scope}
            \clip (0,0.4) circle (1cm);
            \fill[color=gray!35] (1,0) circle (1cm);
          \end{scope}
          \begin{scope}[even odd rule]
            \clip (0,0.4) circle (1cm) (-1,-1) rectangle (2,2);
            \fill[color=gray!70] (1,0) circle (1cm);
          \end{scope}

          \draw (0,0.4) circle (1cm)
                (1,0) circle (1cm);
          \draw (0,-0.6) node[anchor=north east] {$\mathcal{H}_{sel}$}
                (1,-1) node[anchor=north] {$\mathcal{H}_{cons}$};

          \fill (-0.6,0.7) circle (0.05cm);
          \draw (-0.6,0.7) node[anchor=west] {$h^*_{sel}$};

          \fill (0.95,-0.4) circle (0.05cm);
          \draw (0.95,-0.4) node[anchor=west] {$h^*_{cons}$};
        \end{tikzpicture}
      \caption{
        If $h^*_{cons} \notin \mathcal{H}_{sel}$, the active selection may not have enough information to find the optimal hypothesis in $\mathcal{H}_{cons} \backslash \mathcal{H}_{sel}$, the grey area: the active learner focused on $h^*_{sel}$ and did not need information for hypotheses outside $\mathcal{H}_{sel}$.
      }
      \label{fig:condition-subset-a}
    \end{minipage}
    \hfill
    \begin{minipage}[t]{0.47\textwidth}
      \centering
        \begin{tikzpicture}[scale=1, line width=0.3, >=stealth]
          \begin{scope}
            \clip (0,0.4) circle (1cm);
            \fill[color=gray!35] (1,0) circle (1cm);
          \end{scope}
          \begin{scope}[even odd rule]
            \clip (0,0.4) circle (1cm) (-1,-1) rectangle (2,2);
            \fill[color=gray!70] (1,0) circle (1cm);
          \end{scope}

          \draw (0,0.4) circle (1cm)
                (1,0) circle (1cm);
          \draw (0,-0.6) node[anchor=north east] {$\mathcal{H}_{sel}$}
                (1,-1) node[anchor=north] {$\mathcal{H}_{cons}$};

          \fill (-0.6,0.7) circle (0.05cm);
          \draw (-0.6,0.7) node[anchor=west] {$h^*_{sel}$};

          \fill (0.2,0.1) circle (0.05cm);
          \draw (0.2,0.1) node[anchor=west] {$h^*_{cons}$};

          \fill (0.95,-0.4) circle (0.05cm);
          \draw (0.95,-0.4) node[anchor=west] {$h$};
        \end{tikzpicture}
      \caption{
        Even if $h^*_{cons} \in \mathcal{H}_{sel}$, the active selection may not have enough information to find $h^*_{cons}$: there may not be enough information to reject every hypothesis in the grey area.
        Does it know that $h$ is not optimal?
      }
      \label{fig:condition-subset-b}
    \end{minipage}
\end{figure}

This can be a problem in several ways.
If the optimal hypothesis $h^*_{cons}$ is one of the new hypotheses, i.e., if $h^*_{cons} \in \mathcal{H}_{cons} \backslash \mathcal{H}_{sel}$ (Figure~\ref{fig:condition-subset-a}), the active learner has less information to find the best hypothesis in that area than the random sample.
It is likely that there are problems where the hypothesis selected by active learning is worse than the hypothesis selected by random sampling, i.e., where $\mathrm{err}\left(h_{cons,AL,n}\right) > \mathrm{err}\left(h_{cons,RD,n}\right)$, which means that there is no sample reusability.

Even if the optimal hypothesis $h^*_{cons}$ is one of the hypotheses in $\mathcal{H}_{sel}$ -- and even if $h^*_{cons} = h^*_{sel}$ -- active learning might still select the wrong hypothesis (Figure~\ref{fig:condition-subset-b}).
The active selection may not have sufficient information to see that every new hypothesis in $\mathcal{H}_{cons} \backslash \mathcal{H}_{sel}$ is worse: there are problems where the information in the active selection is such that there is a hypothesis $h \in \left(\mathcal{H}_{cons} \backslash \mathcal{H}_{sel}\right)$ that, on the active sample selection, is better than $h^*_{cons}$.
This might happen to random samples as well, of course, but it is more likely to happen with the active selection.
In that case, $\mathrm{err}\left(h_{cons,AL,n}\right) > \mathrm{err}\left(h_{cons,RD,n}\right)$, which means that there is no sample reusability.

Apparently, reusability can not be guaranteed if the consumer can find hypotheses that the selector did not have to consider.
There may be reusability in individual cases, but in general, $\mathcal{H}_{cons} \subseteq \mathcal{H}_{sel}$ is a necessary condition for reusability.

\subsection{Necessary condition 2: $h^*_{sel} \in \mathcal{H}_{cons}$}

A second condition is that the optimal hypothesis $h^*_{sel}$ for the selector should also be in the hypothesis space of the consumer.
Suppose that this is not the case, that $\mathcal{H}_{sel} \supseteq \mathcal{H}_{cons}$, but that $h^*_{sel} \notin \mathcal{H}_{cons}$ (Figure~\ref{fig:condition-incons}).

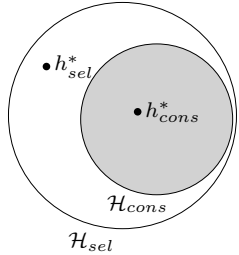
\begin{figure}[h]
  \centering
    \begin{tikzpicture}[scale=1, line width=0.3, >=stealth]
      \fill[color=gray!35] (0.45,0) circle (1cm);

      \draw (0,0.05) circle (1.5cm)
            (0.45,0) circle (1cm);
      \draw (0,-1.4) node[anchor=north east] {$\mathcal{H}_{sel}$}
            (0.2,-0.9) node[anchor=north] {$\mathcal{H}_{cons}$};

      \fill (-1,0.7) circle (0.05cm);
      \draw (-1,0.7) node[anchor=west] {$h^*_{sel}$};

      \fill (0.2,0.1) circle (0.05cm);
      \draw (0.2,0.1) node[anchor=west] {$h^*_{cons}$};
    \end{tikzpicture}
  \caption{
    If $h^*_{sel} \notin \mathcal{H}_{cons}$, the selector has enough samples to show that $h^*_{sel}$ is better than any hypothesis in $\mathcal{H}_{sel}$.
    There may not be enough information to optimise and find $h^*_{cons}$ in $\mathcal{H}_{cons}$.
  }
  \label{fig:condition-incons}
\end{figure}

Then $h^*_{cons}$ was one of the hypotheses that were available to the selector: $h^*_{cons} \in \mathcal{H}_{sel}$.
But it is not the optimal solution in $\mathcal{H}_{sel}$, and that may be a problem.
There will be enough examples to show that $h^*_{sel}$ was better than $h^*_{cons}$, since $h^*_{cons}$ was available to the active learner but was not selected.
But to be selected as $h^*_{cons}$, there should be examples in the sample selection that show that the hypothesis is better than any other $h \in \mathcal{H}_{cons}$.
There is no guarantee that that information is not available: since it was not a question that the selector needed to answer, the examples that are needed to answer the question may not have been selected.
There must be problems where the random sample provides more information near $h^*_{cons}$ than the active selection.
In that case it is likely that $h_{cons,RD,n}$ is closer to $h^*_{cons}$ than $h_{cons,AL,n}$.
This means that $\mathrm{err}\left(h_{cons,AL,n}\right) > \mathrm{err}\left(h_{cons,RD,n}\right)$ and that there is no reusability.

Apparently, reusability can not be guaranteed if the consumer finds a different hypothesis than the selector.
There may be reusability in individual cases, but in general, $h^*_{sel} \in \mathcal{H}_{cons}$ is a necessary condition for reusability.

\subsection{Sufficient conditions?}

The two conditions are necessary to guarantee sample reusability: without $\mathcal{H}_{sel} \supseteq \mathcal{H}_{cons}$ and $h^*_{sel} \in \mathcal{H}_{cons}$ there may be sample reusability in some or even in many problems, but not in all -- if there is any reusability, it is due to luck.
To guarantee reusability the classifiers need to meet these two conditions, and the conditions are quite strong.
The first condition requires that the selector is more powerful than the consumer.
The second condition requires that this extra power is not useful: the selector should not find a solution that is better than the solution of the consumer.
As a result, the conditions can probably only be met by classifiers that are so similar that they produce the same classifier.

The two necessary conditions do much to improve the chance of reusability, but they are still not sufficient to make a guarantee.
The condition $h^*_{sel} \in \mathcal{H}_{cons}$ requires that the selector and the consumer converge to the same hypothesis, but that is only true if there is an infinite sample selection.
In practice, reusability should happen at limited sample sizes.

\newthought
It may be possible to find a condition that guarantees reusability at limited sample size.
Here is a condition that can do this -- although it may be stronger than absolutely necessary.
Consider the situation at a sample size of $n$ samples.
The condition $\mathcal{H}_{sel} \supseteq \mathcal{H}_{cons}$ implies that the selector has access to any hypothesis of the consumer.
Then the best hypothesis of the consumer has an error on the current sample selection that is at least as large as the error of the best hypothesis of the selector:
\begin{align*}
  \mathrm{err}\left(h_{sel,AL,n}, S_{AL,n},W_{AL,n}\right) &\le
  \mathrm{err}\left(h_{cons,AL,n}, S_{AL,n},W_{AL,n}\right) \\
  \mathrm{err}\left(h_{sel,AL,n}, S_{AL,n},W_{AL,n}\right) &\le
  \mathrm{err}\left(h_{cons,AL,n}, S_{AL,n},W_{AL,n}\right) 
\end{align*}
The importance weights in the sample selection make the error on $S_{AL,n}$ an unbiased estimator for the error on unseen data from the true distribution, i.e., $\mathrm{E}\left[\mathrm{err}\left(h,S_{AL,n},W_{AL,n}\right)\right] = \mathrm{err}\left(h\right)$, so the previous inequality can be written as
\[
  \mathrm{err}\left(h_{sel,AL,n}\right) \le
  \mathrm{err}\left(h_{cons,AL,n}\right)
\]
The same holds for random sampling, so
\[
  \mathrm{err}\left(h_{sel,RD,n}\right) \le
  \mathrm{err}\left(h_{cons,RD,n}\right)
\]
Since the active learner is assumed to be functional, the expected error of the classifier selected by self-selection should be better than the expected error with random sampling:
\[
  \mathrm{err}\left(h_{sel,AL,n}\right) \le
  \mathrm{err}\left(h_{sel,RD,n}\right)
\]
but there is only reusability if the classifier of the \textit{consumer} is better with active learning than with random sampling, that is, if
\[
  \mathrm{err}\left(h_{cons,AL,n}\right) \le
  \mathrm{err}\left(h_{cons,RD,n}\right)
\]
One case where this is \textit{guaranteed} is if the expected errors of the selector and consumer hypotheses are the same. Then
\begin{align*}
  \mathrm{err}\left(h_{sel,AL,n}\right) &=
  \mathrm{err}\left(h_{cons,AL,n}\right) \\
  \mathrm{err}\left(h_{sel,RD,n}\right) &=
  \mathrm{err}\left(h_{cons,RD,n}\right) \\
  \mathrm{err}\left(h_{cons,AL,n}\right) & \le
  \mathrm{err}\left(h_{cons,RD,n}\right)
\end{align*}
This is true if $\mathcal{H}_{cons}$ contains both $h_{sel,AL,n}$ and $h_{sel,RD,n}$.

In other words: the hypothesis space of the consumer should not only contain the optimal hypothesis of the selector, but should also contain any intermediate hypotheses (Figure~\ref{fig:condition-path}).
Reusability can be guaranteed if the consumer can follow the same path towards the solution as the selector.

\begin{figure}[h]
  \centering
    \begin{tikzpicture}[scale=1, line width=0.3, >=stealth]
      \fill[color=gray!35] (-0.1,0.2) circle (1.2cm);

      \draw (0,0.05) circle (1.5cm)
            (-0.1,0.2) circle (1.2cm);
      \draw (0,-1.4) node[anchor=north east] {$\mathcal{H}_{sel}$}
            (0.3,-0.97) node[anchor=north] {$\mathcal{H}_{cons}$};

      \fill (-1,0.6) circle (0.05cm);
      \draw (-1,0.6) node[anchor=south west] {$h^*_{cons}$};

      \draw [decorate, decoration={snake, segment length=5mm, amplitude=1.5mm}] (0.0,0.0) -- (-1,0.6);
    \end{tikzpicture}
  \caption{
    One situation where reusability is guaranteed: $\mathcal{H}_{cons}$ should contain all intermediate hypotheses $h_{sel,AL,n}$ on the way to $h^*_{sel}=h^*_{cons}$.
  }
  \label{fig:condition-path}
\end{figure}
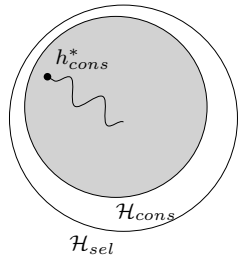

\newthought
This result may not be very useful: it says that you can guarantee reusability if the selector and the consumer that are almost or completely the same.
That is the definition of self-selection.
There is some room between the two required conditions and this sufficient condition, so there might be a less restrictive condition that satisfies $\mathrm{err}\left(h_{cons,AL,n}\right) \le \mathrm{err}\left(h_{cons,RD,n}\right)$ but does not require the two classifiers to be the same.
That, however, might require additional knowledge of the specific classifiers.

An alternative could be to use a more probabilistic approach to the reusability question.
Instead of asking for conditions that can guarantee reusability for every possible problem -- which, as this section shows, leads to a discussion of the few exceptions where it does not work -- it might be enough to find a condition that predicts reusability in most problems, but just not in all.

This would fit in with the proofs of the importance-weighted active learning algorithm itself.
\citet{Beygelzimer2010} do not provide absolute guarantees, but give confidence bounds of the performance of their algorithm relative to random sampling.
It might be possible to provide similar bounds that predict that, with a certain probability, the performance of importance-weighted active learning with foreign-selection is not much worse than the performance of random sampling.
This would probably require strong assumptions about the datasets.

\section{Discussion and conclusion}
\label{sec:discussion}

Active learning is a wonderful idea.
Its promise to deliver better models at a fraction of the cost is very attractive.
But active learning has a dark side -- in fact, rather many dark sides.
The methods that work for the datasets in a study may not work so well on a real dataset, and active learning might give results that are much worse than random sampling.
It is hard to evaluate the quality of the sample selection when it is complete, and it is even harder to predict the result before starting the sampling.
This unpredictability makes it difficult to justify the use of active learning.

Recent developments in importance-weighted active learning have removed some of the problems.
It has a random component, which helps to prevent the missed cluster problem, and it has importance weighting, which helps to remove part of the bias in the sample selection.
The behaviour of the algorithm is also described theoretically: there are proven bounds for the sample complexity, label complexity and the expected error of the active learner.
This makes it one of the best active learning methods developed so far.
Unfortunately, importance-weighted active learning does not solve every problem.

One of the unsolved questions is that of sample reusability, the topic of this paper.
Sample reusability is important in many practical applications of active learning, but it is little-understood.
It would be useful if sample reusability could be predicted, but that is a hard problem.
The only real study on this topic \citep{Tomanek2011} has inconclusive results.
Using uncertainty sampling as the active learning strategy, the study found no pairs of classifiers that always show reusability and did not find a reliable way to predict when reusability would or would not happen.
The study examined a number of hypotheses that could have explained the conditions for reusability, but none of these hypotheses were true.

We investigated the sample reusability in importance-weighted active learning.
The authors of the importance-weighted active learning framework suggest that it produces reusable samples \citep{Beygelzimer2011}.
This is a reasonable idea, because the algorithm solves many of the problems of previous active learning strategies, the importance weighting provides an unbiased sample selection, and therefore this selection might also be more reusable.

However, this paper argued that even importance-weighted active learning does not always produce reusable results.
If the samples selected for one classifier need to be reusable by \textit{every} other classifier, the active sample selection should be at least as good as the random selection.
This is an impossible task, because active learning works by selecting fewer samples than random sampling in some areas of the sample space -- the areas that do not influence the result.
But in general, every sample is interesting to some classifier, so nothing can be left out.

Yes, importance weighting creates an unbiased dataset, so the consumer will always converge to the optimal hypothesis if it is given enough samples.
This makes the sample selection of importance-weighted active learning potentially more reusable than the sample selection of a selection strategy that can not make this guarantee, such as uncertainty sampling.
But in the practical definition of reusability, active learning should also produce a better hypothesis on a \textit{limited} number of samples, and at limited sample sizes importance weighting does not work that well.
It corrects the bias on average, so the expected sample distribution is correct, but individual sample selections are still different from the true distribution.
The expected error of the hypothesis selected with an active sample can be worse than the expected error of the hypothesis selected with a random sample.
Even worse: in some cases, especially at smaller sample sizes, importance weighting introduces so much variability that the results with weighting are worse than the results without.

The results of the practical experiments illustrate that these issues are not only theoretical.
There was reusability in some cases, for specific combinations of classifiers and datasets.
Importance-weighted active learning also seems to have produced selections that are more reusable than those of uncertainty sampling -- although there were also instances where the opposite is true.
However, the many and unpredictable cases where the importance-weighted selections are not reusable make it clear that importance-weighted active learning has not solve the reusability problem.

Are there any certainties in sample reusability?
This paper explored conditions that might guarantee sample reusability, but these conditions were quite strong: to get reusability on all possible datasets, the selector and the consumer should be almost exactly the same.
For guaranteed reusability, the selector should be able to find all hypotheses of the consumer and the consumer should be able to find the optimal hypothesis of the selector.
Even then, if the selector and consumer are not exactly the same, it is possible to find that there is no reusability at smaller sample sizes.
This suggests that true universal reusability is impossible.

\section*{Note}

This paper is a completely revised version of the the \href{http://resolver.tudelft.nl/uuid:af4f9074-774e-4ff9-bab2-b58970b1c990}{Master's thesis} by the first author, Gijs van Tulder, written at the Delft University of Technology in 2012 under the supervision of the second author, Marco Loog \citep{VanTulder2012}.

\bibliographystyle{plainnat}

\end{document}